\documentclass[11pt]{article}

\usepackage[style = alphabetic,citestyle=alphabetic,maxbibnames=99,backend=biber,sorting=nyt,natbib=true,backref=true]{biblatex}
\DefineBibliographyStrings{english}{%
 backrefpage = {Cited on page},
 backrefpages = {Cited on pages},
}
\usepackage[colorlinks,linkcolor = RawSienna, urlcolor  = RedViolet, citecolor = RoyalBlue, anchorcolor = ForestGreen,bookmarks=False]{hyperref}
\usepackage[dvipsnames]{xcolor}
\usepackage{amsthm}
\usepackage{amsmath}
\usepackage{amssymb}
\usepackage{spencercommands}
\usepackage{fullpage,times,url,caption}

\usepackage[shortlabels]{enumitem}
\usepackage{enumerate}
\usepackage{subcaption}

\newtheorem{theorem}{Theorem}[section]

\newtheorem{lemma}[theorem]{Lemma}

\newtheorem{definition}[theorem]{Definition}

\newtheorem{fact}[theorem]{Fact}
\newtheorem{example}[theorem]{Example}

\author{
Spencer Frei \\
UC Berkeley \\
frei@berkeley.edu
\and 
Niladri S. Chatterji  \\
Stanford University \\
niladri@cs.stanford.edu
      \and
Peter L. Bartlett \\
UC Berkeley \\
peter@berkeley.edu
}

\date{\today}

\title{\textbf{Benign Overfitting without Linearity: \\Neural Network Classifiers Trained by Gradient Descent\\ for Noisy Linear Data}}
\begin{document}
\maketitle

\begin{abstract}

Benign overfitting, the phenomenon where interpolating models generalize well in the presence of noisy data, was first observed in neural network models trained with gradient descent. 
To better understand this empirical observation, we consider the generalization error of two-layer neural networks trained to interpolation by gradient descent on the logistic loss following random initialization.
We assume the data comes from well-separated class-conditional log-concave distributions and allow for a constant fraction of the training labels to be corrupted by an adversary.  We show that in this setting, neural networks exhibit benign overfitting: they can be driven to zero training error, perfectly fitting any noisy training labels, and simultaneously achieve 
minimax optimal test error. 
In contrast to previous work on benign overfitting that require linear or kernel-based predictors, our analysis holds in a setting where both the model and learning dynamics are fundamentally nonlinear.

\end{abstract}
\section{Introduction}
Trained neural networks have been shown to generalize well to unseen data even when trained to interpolation (that is, vanishingly small training loss) on training data with significant label noise~\citep{zhang2017rethinkinggeneralization,belkin2019reconciling}.   This empirical observation is surprising as it appears to violate long standing intuition from statistical learning theory that the greater the capacity of a model to fit randomly labelled data, the worse the model's generalization performance on test data will be.  This conflict between theory and practice has led to a surge of theoretical research into the generalization performance of interpolating statistical models to see if this `benign overfitting' phenomenon can be observed in simpler settings that are more amenable to theoretical investigation.  We now understand that benign overfitting can occur in many classical statistical settings, including linear regression~\citep{hastie2019surprises,bartlett2020.benignoverfitting.pnas,muthukumar2020harmless,negrea2020defense,tsigler2020benign,chinot2020robustness_min,chatterji2021interplay}, sparse linear regression~\citep{koehler2021uniform,chatterji2021foolish,li2021minimum,wang2021tight}, logistic regression~\citep{montanari2019generalization,chatterji2020linearnoise,liang2020precise,muthukumar2021classification,wang2021benignmulticlass,minsker2021minimaxinterpolation}, and kernel-based estimators~\citep{belkin2018overfittingperfectfitting,mei2019generalization,liang2020just,liang2020multipledescent}, among others, and our understanding of when and why this phenomenon occurs in these settings is rapidly increasing. And yet, for the class of models from which the initial motivation for understanding benign overfitting arose---trained neural networks---we understand remarkably little. 

In this work,   
we consider the class of two-layer networks with smoothed leaky ReLU activations trained on data coming from a high-dimensional linearly separable dataset where a constant fraction of the training labels can be adversarially corrupted~\citep{kearns.agnostic}.  We demonstrate that networks trained by standard gradient descent on the logistic loss in this setting exhibit benign overfitting: they can be driven to zero loss, and thus interpolate the noisy training data, and simultaneously achieve 
minimax optimal generalization error.

Our results follow by showing that the training loss can be driven to zero while the expected normalized margin for clean data points is large.  The key technical ingredient of the proof for both of these claims is a `loss ratio bound': we show that the gradient descent dynamics ensure that the loss of each example decreases at roughly the same rate throughout training.  This ensures that the noisy points cannot have an outsized influence on the training dynamics, so that we can have control over the normalized margin for clean data points throughout training.  At a high-level, this is possible because the data is high-dimensional, which ensures that all data points are roughly mutually orthogonal. 

Our results hold for finite width networks, and since the logistic loss is driven to zero, the weights traverse far from their randomly initialized values.  As a consequence, 
this shows benign overfitting behavior in trained neural networks beyond the kernel regime~\citep{jacot2018ntk}.

\subsection{Related Work}
A number of recent works have characterized the generalization performance of interpolating models.  Most related to ours are those in the classification setting.~\citet{chatterji2020linearnoise} study 
the high-dimensional sub-Gaussian mixture model setup we consider here, where labels can be corrupted adversarially, and analyze the performance of the maximum margin linear classifier.  They do so by utilizing recent works that show that the weights found by unregularized gradient descent on the logistic loss asymptotically approach the maximum margin classifier for linearly separable data~\citep{soudry2018implicitbias,jitelgarsky2019implicit}.  Our proof techniques can be viewed as an extension of some of the techniques developed by Chatterji and Long in the logistic regression setting to two-layer neural networks.~\citet{muthukumar2021classification} study the behavior of the overparameterized max-margin classifier in a discriminative classification model with label-flipping noise, by connecting the behavior of the max-margin classifier to the ordinary least squares solution.  They show that under certain conditions, all training data points become support vectors of the maximum margin classifier \citep[see also,][]{hsu2021proliferation}. Following this, \citet{wang2021binary} and \citet{cao2021benign} analyze the behavior of the overparameterized max-margin classifier in high dimensional mixture models by exploiting the connection between the max-margin classifier and the OLS solution. In contrast with these works, we consider the generalization performance of an interpolating nonlinear neural network.

A key difficulty in establishing benign overfitting guarantees for trained neural networks lies in demonstrating that the neural network can interpolate the data.~\citet{brutzkus2018sgd} study SGD on two-layer networks with leaky ReLU activations and showed that for linearly separable data, stochastic gradient descent on the hinge loss will converge to zero training loss.  They provided guarantees for the test error provided the number of samples is sufficiently large relative to the input dimension and the Bayes error rate is zero, but left open the question of what happens when there is label noise or when the data is high-dimensional.~\citet{frei2021provable} show that for linear separable data with labels corrupted by adversarial label noise~\citep{kearns.agnostic}, SGD on the logistic loss of two-layer leaky ReLU networks achieves test error that is at most a constant multiple of the square root of the noise rate under mild distributional assumptions. 
However, their proof technique did not allow for the network to be trained to interpolation. 
In contrast, we allow for the network to be trained to arbitrarily small loss and hence interpolate noisy data.  In principle, this could allow for the noisy samples to adversely influence the classifier, but we show this does not happen. 

A series of recent works have exploited the connection between overparameterized neural networks and an infinite width approximation known as the neural tangent kernel (NTK)~\citep{jacot2018ntk,allenzhu2019convergence,zou2019gradient,du2019-1layer,arora2019exact,soltanolkotabi2017overparameterized}.  These works show that for a certain scaling regime of the initialization, learning rate, and width of the network, neural networks trained by gradient descent behave similarly to their linearization around random initialization and can be well-approximated by the NTK.  The near-linearity simplifies much of the analysis of neural network optimization and generalization.  Indeed, a number of recent works have characterized settings in which neural networks in the kernel regime can exhibit benign overfitting~\citep{liang2020multipledescent,montanari2021interpolationphasetransition}.

Unfortunately, the kernel approximation fails to meaningfully capture a number of aspects of neural networks trained in practical settings, such as the ability to learn features~\citep{yang2021feature}, so that previous kernel-based approaches for understanding neural networks provide a quite restricted viewpoint for understanding neural networks in practice.  By contrast, in this work, we develop an analysis of benign overfitting in finite width neural networks trained for many iterations on the logistic loss.  We show that gradient descent drives the logistic loss to zero so that the weights grow to infinity, far from the near-initialization region where the kernel approximation holds, while the network simultaneously maintains a positive margin on clean examples.  This provides the first guarantee for benign overfitting that does not rely upon an effectively linear evolution of the parameters.   

Finally, we note in a concurrent work~\citet{cao2022benignconvolutional} characterize the generalization performance of interpolating two-layer convolutional neural networks.  They consider a distribution where input features consist of two patches, a `signal' patch and a `noise' patch, and binary output labels are a deterministic function of the signal patch.  They show that if the signal-to-noise ratio is larger than a threshold value then the interpolating network achieves near-zero test error, while if the signal-to-noise ratio is smaller than the threshold then the interpolating network generalizes poorly. 
There are a few key differences in our results.  
First, our setup allows for a constant fraction of the training labels to be random, while in their setting the training labels are a deterministic function of the input features.  Achieving near-zero training loss in our setting thus requires overfitting to noisy labels, in contrast to their setting where such overfitting is not possible.  
Second, they require the input dimension to be at least as large as $m^2$ (where $m$ is the number of neurons in the network), while our results do not make any assumptions on the relationship between the input dimension and the number of neurons in the network.  


\section{Preliminaries}

In this section we introduce the assumptions on the data generation process, the neural network architecture, and the optimization algorithm we consider. 

\subsection{Notation}
We denote the $\ell^2$ norm of a vector $x\in \R^p$ by $\snorm{x}$.  For a matrix $W\in \R^{m\times p}$, we use $\snorm{W}_F$ to denote its Frobenius norm and $\snorm{W}_2$ to denote its spectral norm, and we denote its rows by $w_1, \dots, w_m$.  For an integer $n$, we use the notation $[n]$ to refer to the set $[n] = \{1, 2, \dots, n\}$.  

\subsection{Setting}
We shall let $C>1$ denote a positive absolute constant, and our results will hold for all values of $C$ sufficiently large.  We consider a mixture model setting similar to one previously considered by~\citet{chatterji2020linearnoise}, defined in terms of a joint distribution $\pnoise$ over $(x,y)\in \R^p \times \{\pm 1\}$.  Samples from this distribution can have noisy labels, and so we will find it useful to first describe a `clean' distribution $\pclean$ and then define the true distribution $\pnoise$ in terms of $\pclean$.  Samples $(x,y)$ from $\pnoise$ are constructed as follows:
\begin{enumerate}
    \item Sample a clean label $\tilde y\in \{\pm 1\}$ uniformly at random, $\tilde y \sim \mathsf{Uniform}(\{+1, -1\})$.
    \item Sample $z\sim \pclust$ where
    \begin{itemize}
        \item $\pclust= \pclust^{(1)}\times \cdots \times \pclust^{(p)}$ is a product distribution whose marginals are all mean-zero with sub-Gaussian norm at most one;
        \item $\pclust$ is  a $\lambda$-strongly log-concave distribution over $\R^p$ for some $\lambda>0$;\footnote{That is, $z\sim \pclust$ has a probability density function $p_z$ satisfying $p_z(x) = \exp(-U(x))$ for some  convex function $U: \R^p \to \R$ such that $\nabla^2 U(x) - \lambda I_p$ is positive semidefinite.}
        \item for some $\kappa >0$, it holds that $\E_{z\sim \pclust}[\snorm{z}^2] \geq \kappa p$.
    \end{itemize}
    \item Generate $\tilde x = z + \tilde y \mu$.
    \item Then, given a noise rate $\eta \in [0,1/C]$, $\pnoise$ is any distribution over $\R^p \times \{\pm 1\}$ such that the marginal distribution of the features for $\pnoise$ and $\pclean$ coincide, and the total variation distance between the two distributions satisfies $d_{\mathsf{TV}}(\pclean, \pnoise) \leq \eta$.  Equivalently, $\pnoise$ has the same marginal distribution over $x$ as $\pclean$, but a sample $(x,y)\sim \pnoise$ has label equal to $\tilde y$ with probability $1-\eta(x)$ and has label equal to $-\tilde y$ with probability $\eta(x)$, where $\eta(x)\in [0,1]$ satisfies $\E_{x\sim \pnoise}\left[\eta(x)\right]\le \eta$. 
\end{enumerate}

We note that the above assumptions coincide with those used by~\citet{chatterji2020linearnoise} in the linear setting with the exception of the introduction of an assumption of $\lambda$ strong log-concavity that we introduce.   This assumption is needed so that we may employ a concentration inequality for Lipschitz functions for strongly log-concave distributions.    We note that variations of this data model have also been studied recently~\citep{wang2021binary,liang2021interpolating,wang2021importance}.

One example of a cluster distribution which satisfies the above assumptions is the (possibly anisotropic) Gaussian.
\begin{example}
If $\pclust = \mathsf{N}(0, \Sigma)$, where $\snorm{\Sigma}_2 \leq 1$ and $\snorm{\Sigma^{-1}}\leq 1/\kappa$, and each of the labels are flipped independently with probability $\eta$, then all the properties listed above are satisfied.
\end{example}

Next, we introduce the neural network architecture and the optimization algorithm.  We consider one-hidden-layer neural networks of width $m$ that take the form
\[ f(x; W) := \summ j m a_j \phi(\sip{w_j}{x}),\]
where we denote the input $x\in \R^p$ and emphasize that the network is parameterized by a matrix $W\in \R^{m\times p}$ corresponding to the first layer weights $\{w_j\}_{j=1}^m$.  
The network's second layer weights $\{a_j\}_{j=1}^m$ are initialized $a_j\iid \Unif(\{1/\sqrt m, -1/\sqrt m\})$ and fixed at their initial values.   
We assume the activation function $\phi$ satisfies $\phi(0)=0$ and is strictly increasing,  1-Lipschitz, and $H$-smooth, that is, it is twice differentiable almost everywhere and there exist $\gamma,H >0$ such that
\[ 0 < \gamma \leq \phi'(z) \leq 1,\quad \text{and}\quad |\phi''(z)| \leq H,\,\, \forall z\in \R.\]
An example of such a function is a smoothed leaky ReLU activation, 
\begin{equation}
    \phi_{\mathrm{SLReLU}}(z) = \begin{cases} z - \f{1-\gamma}{4H}, & z \geq 1/H,\\
    \f{1-\gamma}4H z^2 + \f{1+\gamma}{2} z, & |z|\leq 1/H,\\
    \gamma z - \f{1-\gamma}{4H}, & z \leq -1/H.\end{cases}
\end{equation}
As $H\to \infty$, $\phi_{\mathrm{SLReLU}}$ approximates the leaky ReLU activation $z\mapsto \max(\gamma z, z)$.  We shall refer to functions $\phi$ satisfying the above properties as $\gamma$-leaky, $H$-smooth activations.

We assume access to a set of samples $S = \{(x_i, y_i)\}_{i=1}^n\iid \pnoise^n$.  We denote by $\calC\subset [n]$ the set of indices corresponding to samples with \textit{clean} labels, and $\calN$ as the set of indices corresponding to \textit{noisy} labels, so that $i\in \calN$ implies $(x_i, y_i)\sim \pnoise$ is such that $y_i = -\tilde y_i$ using the notation above.

Let $\ell(z) = \log(1+\exp(-z))$ be the logistic loss, and denote the empirical and population risks under $\ell$ by
\[ \hat L(W) := \f 1n \summ i n \ell(y_i f(x_i; W))\quad\text{and}\quad L(W) := \E_{(x,y)\sim \pnoise}\left[ \ell(y f(x; W))\right].\]
We will also find it useful to treat the function $-\ell'(z) = 1/(1+\exp(z))$ as a loss itself: since $\ell$ is convex and decreasing, $-\ell'$ is non-negative and decreasing and thus can serve as a surrogate for the 0-1 loss.  This trick has been used in a number of recent works on neural network optimization~\citep{cao2019generalization,frei2019resnet,jitelgarsky20.polylog,frei2021provable}.  To this end, we introduce the notation,
\begin{align*}
    g(z) &:= - \ell'(z)\quad \text{and}\quad    \hat G(W) := \f 1 n \summ i n g(y_i f(x_i; W)).
\end{align*}
We 
also introduce notation to refer to the function output and the surrogate loss $g$ evaluated at samples for a given time point,
\[ \fit it := f(x_i; \Wt t)\quad \text{and} \quad  \lpit it := g\big(y_i \fit it \big).\]

We initialize the first layer weights independently for each neuron according to standard normals $[\Wt 0]_{i,j}\iid \sfN(0, \sinit^2)$, where $\sinit^2$ is the initialization variance.  The optimization algorithm we consider is unregularized full-batch gradient descent on $\hat L(W)$ initialized at $\Wt 0$ with fixed step-size $\alpha >0$ which has updates
\[ \Wt {t+1} = \Wt t - \alpha \nabla \hat L(\Wt t).\]

Given a failure probability $\delta \in (0,1/2)$, we make the following assumptions on the parameters in the paper going forward:
\begin{enumerate}[label=(A\arabic*)]
    \item \label{a:samples}Number of samples $n\geq C\log(1/\delta)$.
    \item \label{a:dimension}Dimension $p\geq C \max\{n \snorm{\mu}^2, n^2 \log(n/\delta)\}$.
    \item \label{a:norm.mu}Norm of the mean satisfies $\snorm{\mu}^2 \geq C \log(n/\delta)$.
    \item \label{a:noiserate}Noise rate $\eta \leq 1/C$.
    \item \label{a:stepsize}Step-size $\alpha \leq \left(C\max\left\{1,\frac{H}{\sqrt m}\right\}p^2\right)^{-1}$, where $\phi$ is $H$-smooth.
    \item \label{a:sinit}Initialization variance satisfies $\sinit  \sqrt{mp} \leq \alpha$. 
\end{enumerate}

Assumptions~\ref{a:samples},~\ref{a:dimension}, and~\ref{a:norm.mu} above have previously appeared in~\citet{chatterji2020linearnoise} and put a constraint on how the number of samples, dimension, and cluster mean separation can relate to one another.  One regime captured by these assumptions is when the mean separation satisfies $\snorm{\mu}=\Theta(p^\beta)$, where $\beta \in (0, 1/2)$ and $p \geq C\max\{ n^\f{1}{1-2\beta}, n^2 \log(n/\delta)\}$. 
Assumption~\ref{a:sinit} ensures that the first step of gradient descent dominates the behavior of the neural network relative to that at initialization; this will be key to showing that the network traverses far from initialization after a single step, which we show in Proposition~\ref{proposition:non.ntk}.   We note that our analysis holds for neural networks of arbitrary width $m\geq 1$.  

\section{Main Result}
Our main result is that when a neural network is trained on samples from the distribution $\pnoise$ described in the previous section, it will exhibit benign overfitting: the network achieves arbitrarily small logistic loss, and hence interpolates the noisy training data, and simultaneously achieves test error close to the noise rate.

\begin{theorem}\label{theorem:main}
For any $\gamma$-leaky, $H$-smooth activation $\phi$, and for all $\kappa \in (0,1)$, $\lambda > 0$, there is a $C>1$ such that provided Assumptions~\ref{a:samples} through~\ref{a:sinit} are satisfied, the following holds.  
For any $0<\eps < 1/2n$, by running gradient descent for $T\geq C \hat L(\Wt 0)/\left( \snorm{\mu}^2 \alpha \eps^{2}\right)$ iterations, with probability at least $1-2\delta$ over the random initialization and the draws of the samples, the following holds:
\begin{enumerate}
    \item All training points are classified correctly  and the training loss satisfies $\hat{L}(W^{(T)}) \le \eps$.
\item The test error satisfies
\begin{align*} 
\P_{(x,y)\sim \pnoise} \big[y \neq \sgn(f(x; \Wt {T}))\big] &\leq \eta + 2 \exp\l( - \f{n\snorm{\mu}^4 }{Cp} \r).
\end{align*}
\end{enumerate}
\end{theorem}


Theorem~\ref{theorem:main} shows that neural networks trained by gradient descent will exhibit benign overfitting: the logistic loss can be driven to zero so that the network interpolates the noisy training data, and the trained network will generalize with classification error close to the noise rate $\eta$ provided $n\snorm{\mu}^4\gg p$.  Note that when $\pclust = \mathsf N(0,I)$,~\citet[Appendix B]{giraud2019partial} showed that in the noiseless case ($\eta=0$), the minimax test error is at least $c \exp(-c' \min(\snorm{\mu}^2, n \snorm{\mu}^4/p))$ for some absolute constants $c,c'>0$.  
In the setting of random classification noise, where labels are flipped with probability $\eta$ (i.e., $\eta(x) = \eta$ for all $x$), this implies that the minimax test error is at least $\eta + c\exp(-c' \min(\snorm{\mu}^2, n \snorm{\mu}^4/p))$.  
By Assumption~\ref{a:norm.mu}, $\snorm{\mu}^2 > n \snorm{\mu}^4/p$, so that the test error in Theorem~\ref{theorem:main} is minimax optimal up to constants in the exponent in the setting of random classification noise.

We briefly also compare our results to margin bounds in the literature. Note that even if one could prove that the training data is likely to be separated by a large margin, the bound of Theorem~\ref{theorem:main} approaches the noise rate faster than the standard margin bounds \citep{vapnik1999nature,shawe1998structural}.

We note that our results do not require many of the assumptions typical in theoretical analyses of neural networks: we allow for networks of arbitrary width; we permit arbitrarily small initialization variance; and we allow for the network to be trained for arbitrarily long.  In particular, we wish to emphasize that the optimization and generalization analysis used to prove Theorem~\ref{theorem:main} does not rely upon the neural tangent kernel approximation.  One way to see this is that our results cover finite-width networks and require $\snorm{\Wt t}\to \infty$ as $\eps \to 0$ since the logistic loss is never zero.   In fact, for the choice of step-size and initialization variance given in Assumptions~\ref{a:stepsize} and~\ref{a:sinit}, the weights travel far from their initial values after a single step of gradient descent, as we show in Proposition~\ref{proposition:non.ntk} below.

\begin{restatable}{proposition}{ntklower}\label{proposition:non.ntk}
Under the settings of Theorem~\ref{theorem:main}, we have for some absolute constant $C>1$ with probability at least $1-2\delta$ over the random initialization and the draws of the samples,
\[ \f{ \snorm{\Wt 1 - \Wt 0}_F }{\snorm{\Wt 0}_F} \geq \f{ \gamma \snorm{\mu}}{C}.\]
\end{restatable}
The proof for Proposition~\ref{proposition:non.ntk} is provided in Appendix~\ref{app:prop_proof}.  

\section{Proof of Theorem~\ref{theorem:main}}\label{sec:thmproof}
In this section we will assume that Assumptions~\ref{a:samples} through~\ref{a:sinit} are in force for a large constant $C>1$. 

Theorem~\ref{theorem:main} consists of two claims: the first is that the training loss can be made arbitrarily small despite the presence of noisy labels, and the second is that at the same time, the test error of the trained neural network is close to the noise rate when $n\snorm{\mu}^4/p \gg 1$. Both of these claims will be established via a series of lemmas. All of these lemmas are proved in Appendix~\ref{app:theorem_proof}.

The test error bound will follow by establishing a lower bound for the \textit{expected normalized margin} on clean points from the $+\mu$ and $-\mu$ clusters.  We do so in the following lemma which leverages the fact that $\pclust$ is $\lambda$-strongly log-concave.

\begin{restatable}{lemma}{marginlowerfirst}\label{lemma:lipschitz.concentration}
    There exists a universal constant $c>0$, depending only on $\lambda$, such that if $W\neq 0$,
\[     \P( y f(x;W) < 0) \leq \eta + \exp \l( - c \l( \f{ 0 \vee \E[f(\mu+z;W)]}{\snorm{W}_2}\r)^2 \r) + \exp \l( - c \l( \f{ 0 \vee \E[-f(-\mu+z;W)]}{\snorm{W}_2}\r)^2 \r).\]
\end{restatable}

Lemma~\ref{lemma:lipschitz.concentration} shows that in order to prove the test error is near the noise rate, it suffices to prove a lower bound on the unnormalized margin on test samples coming from the $+\mu$ and $-\mu$ clusters as well as an upper bound on the spectral norm of the weights.  To derive such bounds, we first need to introduce a number of structural results about the samples and the neural network objective function.  The first such result concerns the norm of the weights at initialization.  

\begin{restatable}{lemma}{normbound}\label{lemma:initialization.norm} 
There is a universal constant $C_0>1$ 
such that with probability at least $1-\delta$ over the random initialization, 
\[ \snorm{\Wt 0}_F^2 \leq \f 32 \sinit^2mp \quad \text{and} \quad \snorm{\Wt 0}_2 \leq C_0 \sinit (\sqrt m + \sqrt p).\]
\end{restatable}

Our next structural result characterizes some properties of random samples from the distribution. It was proved in \citet[][Lemma~10]{chatterji2020linearnoise} and is a consequence of Assumptions~\ref{a:samples} through~\ref{a:noiserate}. 
\begin{lemma}\label{lemma:sample.facts}
For all $\kappa >0$, there is $C_1>1$ such that for all $c'>0$, for all large enough $C$, with probability at least $1-\delta$ over $\pnoise^n$, the following hold:
\begin{enumerate}[label=E.\arabic*]
    \item  \label{eq:xk.norm} For all $k \in [n]$,
    \begin{align*}
      p/C_1 \leq \snorm{x_k}^2 \leq C_1 p.  
    \end{align*}
    \item \label{eq:xi.xk.innerproduct} For all $i \neq j \in [n]$,
    \begin{align*}
        |\sip{x_i}{x_j}| \leq C_1(\snorm{\mu}^2 + \sqrt{p\log(n/\delta)}).
    \end{align*}
    \item \label{eq:mu.dot.xk.clean}For all $k \in \calC$,
    \begin{align*}
        |\sip{\mu}{y_k x_k} - \snorm{\mu}^2| \leq \snorm{\mu}^2/2.
    \end{align*}
    \item \label{eq:mu.dot.xk.noisy} For all $k \in \calN$,
    \begin{align*}
        |\sip{\mu}{y_k x_k} - (-\snorm{\mu}^2)| \leq \snorm{\mu}^2/2.
    \end{align*}
    \item   \label{eq:noisy.samples.ub} The number of noisy samples satisfies $|\calN|/n \leq \eta + c'$.
\end{enumerate}
\end{lemma}

\begin{definition}
If the events in Lemma~\ref{lemma:initialization.norm} and Lemma~\ref{lemma:sample.facts} occur, let us say that we have a \emph{good run}.
\end{definition}

Lemmas~\ref{lemma:initialization.norm} and~\ref{lemma:sample.facts} show that a good run occurs with probability at least $1-2\delta$. In what follows, we will assume that a good run occurs. 





We next introduce a number of structural lemmas concerning the neural network optimization objective.  The first concerns the smoothness of the network in terms of the first layer weights. 
\begin{restatable}{lemma}{smoothness}\label{lemma:nn.smooth}
For an $H$-smooth activation $\phi$ and any $W, V\in \R^{m\times p}$ and $x\in \R^p$,
\[ |f(x; W) - f(x; V ) - \sip{\nabla f(x; V )}{W-V}| \leq \frac{ H \snorm{x}^2}{2\sqrt m} \|W-V\|_2^2.\]
\end{restatable}

In the next lemma, we provide a number of smoothness properties of the empirical loss. 
\begin{restatable}{lemma}{losssmoothness}\label{lemma:loss.smooth}
For an $H$-smooth activation $\phi$ and any $W, V \in \R^{m\times p}$, 
on a good run it holds that 
\[ \frac{1}{\sqrt{C_1 p}}\snorm{\nabla \hat L(W)}_F \leq  \hat G(W) \leq \hat L(W) \wedge 1,\]
where 
$C_1$ is the constant from Lemma~\ref{lemma:sample.facts}.  Additionally,
\[ \snorm{\nabla \hat L(W) - \nabla \hat L(V)}_F\leq   C_1p \left (  1 + \f{ H}{\sqrt m}  \right ) \snorm{W-V}_2.\]
\end{restatable}

Our next structural result is the following lemma that characterizes the pairwise correlations of the gradients of the network at different samples. 

\begin{restatable}{lemma}{gradientcorrelation}
\label{lemma:nn.grad.ip.identity}
Let $C_1>1$ be the constant from Lemma~\ref{lemma:sample.facts}.  For a $\gamma$-leaky, $H$-smooth activation $\phi$, on a good run, 
we have the following. 
\begin{enumerate}
\item[(a)]For any $i,k\in [n]$, $i\neq k$, and any $W\in \R^{m\times d}$, we have
\[ 
|\sip{\nabla f(x_i; W)}{\nabla f(x_k; W)}| \leq C_1 \l( \snorm{\mu}^2 + \sqrt{p \log(n/\delta)} \r).
\]
\item[(b)]For any $i\in [n]$ and any $W\in \R^{m\times d}$, we have
\[ \f{\gamma^2 p }{C_1} \leq \snorm{\nabla f(x_i; W)}_F^2 \leq C_1 p.\]
\end{enumerate}
\end{restatable}

In the regime where $\snorm{\mu}^2 = o(p)$, Lemma~\ref{lemma:nn.grad.ip.identity} shows that the gradients of the network at different samples are roughly orthogonal as the pairwise inner products of the gradients are much smaller than the norms of each gradient.  This mimics the behavior of the samples $x_i$ established in Lemma~\ref{lemma:sample.facts}. 

Our final structural result establishes that the ratio of the sigmoid losses $\lpit it := -\ell'\big(y_i f(x_i;\Wt t)\big)$ are uniformly bounded throughout training.  In particular, it implies that the clean examples $(i\in \calC)$ and the noisy examples $(i\in \calN$) have sigmoid losses which are not too far from each other, and that this holds throughout all of training.  This lemma plays a central role in our proof of Theorem~\ref{theorem:main}, for both the training error and the test error bounds, and extends the results of~\citet{chatterji2020linearnoise} from the logistic regression setting to the two-layer neural network setting.

\begin{restatable}{lemma}{lossratio}\label{lemma:loss.ratio}
For a $\gamma$-leaky, $H$-smooth activation $\phi$, there is an absolute constant $C_r=16 C_1^2/\gamma^{2}$ such that on a good run,
provided $C>1$ is sufficiently large, 
we have for all $t\geq 0$,
\begin{align*}\label{eq:loss.ratio.bound}
    \max_{i,j \in [n]} \frac{\lpit it}{\lpit jt} \le C_r .
\end{align*}
\end{restatable}

With these structural results in place, we can now begin to prove a lower bound for the normalized margin on test points.  To do so, our first step is to characterize the change in the unnormalized margin $\tilde y[f(x; \Wt {t+1}) - f(x; \Wt t)]$ from time $t$ to time $t+1$ for an independent test sample $(x,y)$, when $x$ comes from the $+\mu$ cluster and when $x$ comes from the $-\mu$ cluster.

\begin{restatable}{lemma}{marginincrease}\label{lemma:margin.test.sample}
For a $\gamma$-leaky, $H$-smooth activation $\phi$, there is an absolute constant $C_\mu>0$ such that provided $n \snorm \mu^4 \geq C_\mu p$, on a good run, provided $C>1$ is sufficiently large, it holds for any $s \in \N \cup \{0\}$,
\begin{align*}
      \E[ f(\mu+z; \Wt {s+1}) - f(\mu+z;\Wt s)] &\geq \f{ \alpha \gamma^2 \snorm{\mu}^2}{8} \hat G(\Wt s),\quad \text{and} \\
      \E[ -f(-\mu+z; \Wt {s+1}) - (-f(-\mu+z;\Wt s))] &\geq \f{ \alpha \gamma^2 \snorm{\mu}^2}{8} \hat G(\Wt s).
\end{align*}
\end{restatable}

Lemma~\ref{lemma:margin.test.sample} shows that provided the cluster mean is large enough so that $\snorm{\mu}^4 \gtrsim p/n$, the unnormalized margin on test examples will increase in expectation from time $s$ to $s+1$.  Notably,~\citet{giraud2019partial} showed that for the regime $\snorm \mu^2 \lesssim p/n$ that we consider here (per Assumption~\ref{a:dimension}), vanishing clean error is impossible when $\snorm \mu^4 = o(p/n)$, which makes the assumption in the lemma a necessary one for learnability in the setting we consider.

We can now provide some insight into the proof of Lemma~\ref{lemma:margin.test.sample}, focusing on the $x=\mu+z$ case (the $x=-\mu+z$ case follows similarly).  Using Lemma~\ref{lemma:nn.smooth}, if we define the quantity
\[ \xi_{i, z, s} := \f 1 m \summ j m \phi'(\sip{\wt s_j}{x_i}) \phi'(\sip{\wt s_j}{\mu+z}) \in [\gamma^2, 1], \]
then recalling the notation $\lpit is := -\ell'\big(y_i f(x_i;\Wt s)\big) \in (0,1)$, we have
\begin{align*}
    \E[ (f(\mu+z; \Wt {s+1}) - f(\mu+z;\Wt s))] &\geq \f \alpha n \summ i n \lpit is \l[ \E[\xi_{i,z,s} \sip{y_ix_i}{\mu+z}] - \f{ H C_1 p \alpha}{2 \sqrt m} \E[\snorm{\mu+z}^2] \r]. \numberthis \label{eq:margin.increase.time.s.mu.cluster.main}
\end{align*}
Ignoring the $z$-dependence of the term $\xi_{i,z,s}$ for the moment, this inequality suggests that if $\sip{y_i x_i}{\mu+z}$ is always bounded from below by a strictly positive constant, then the margin on the test sample $x=\mu+z$ will increase.  Unfortunately, even when ignoring this dependence of $\xi_{i,z,s}$, the presence of noisy labels will cause some of the $\sip{y_i x_i}{\mu+z}$ terms appearing above to be negative, allowing for the possibility that the unnormalized margin decreases on a test sample $x=\mu+z$.  If the losses $g(y_i f(x_i; \Wt t))$ for (noisy) samples satisfying $\sip{y_i x_i}{\mu+z}<0$ are particularly large relative to the losses $g(y_{i'} f(x_{i'}; \Wt t))$ for (clean) samples satisfying $\sip{y_{i'} x_{i'}}{\mu+z} >0$, then this indeed presents a problem.   This is where Lemma~\ref{lemma:loss.ratio} comes in: since the $g$ losses are of the same order for \textit{all} samples, if the fraction of noisy labels is not too large, one can ignore the effect of the noisy labels which contribute negative terms to the sum, and eventually show that the term on the right-hand side of~\eqref{eq:margin.increase.time.s.mu.cluster.main} is strictly positive.  This constitutes the core of the proof we show in the appendix.

The only remaining piece for the \textit{normalized} margin lower bound is an upper bound on the norm of the iterates $\Wt t$.  The following lemma provides an upper bound on the Frobenius norm of the weights.  The proof utilizes the loss ratio bound from Lemma~\ref{lemma:loss.ratio} to get a sharper upper bound than what one would get from attempting to use a standard triangle inequality. 

\begin{restatable}{lemma}{parameternormbound}\label{l:refined_grad_norm_upper_bound} There is an absolute constant $C_2>1$
such that for $C>1$ sufficiently large, on a good run we have that
for all $t\ge 0$,
\begin{align*}
    \lv W^{(t)}\rv_F \le \lv W^{(0)}\rv_F +C_2 \alpha \sqrt{\frac{p}{n}} \sum_{s=0}^{t-1}\hat{G}(W^{(s)}).
\end{align*}
\end{restatable}

With the lower bound on the increment of the unnormalized margin from Lemma~\ref{lemma:margin.test.sample} and the tightened gradient norm bound of Lemma~\ref{l:refined_grad_norm_upper_bound} established, we can now derive a lower bound on the normalized margin.  Note that this lower bound on the normalized margin in conjunction with Lemma~\ref{lemma:lipschitz.concentration} results in the test error bound for the main theorem.
\begin{restatable}{lemma}{normalizedmarginbound}\label{lemma:normalized.margin.lb}
For a $\gamma$-leaky, $H$-smooth activation $\phi$, there exists an absolute constant $C_\mu>0$ such that provided $n \snorm \mu^4 \geq C_\mu p$, on a good run, provided $C>1$ is sufficiently large, it holds for any $t \geq 1$,
\[ \f{ \E[f(\mu+z; \Wt t)] \wedge \E[-f(-\mu+z;\Wt t)]}{\snorm{\Wt t}_F} \geq \f{\gamma^2 \snorm{\mu}^2 \sqrt n }{ 32\max(\sqrt{C_1}, C_2) \sqrt {p}},\]
where $C_1$ and $C_2$ are the constants from Lemma~\ref{lemma:sample.facts} and Lemma~\ref{l:refined_grad_norm_upper_bound}, respectively. 
\end{restatable}

Since Lemma~\ref{lemma:normalized.margin.lb} provides a positive margin on clean test points, we have by Lemma~\ref{lemma:lipschitz.concentration} a guarantee that the neural network achieves classification error on the noisy distribution close to the noise level.  The only remaining part of Theorem~\ref{theorem:main} that remains to be proved is that the training loss can be driven to zero.  This is a consequence of the following lemma, the proof of which also crucially relies upon the loss ratio bound of Lemma~\ref{lemma:loss.ratio}.  

\begin{restatable}{lemma}{gradientnormlowerbound}\label{lemma:optimization.guarantee}
For a $\gamma$-leaky, $H$-smooth activation $\phi$, provided $C>1$ is sufficiently large, then on a good run we have for all $t\geq 0$,
\[ \snorm{\nabla \hat L(\Wt t)}_F \geq \frac{\gamma \lv \mu \rv}{4}   \hat G(\Wt t).\]
Moreover, any $T\in \N$,
\[ \f 1 n \summ i n \ind\big (y_i \neq \sgn(f(x_i; \Wt {T-1})) \big) \leq 2 \hat G(\Wt {T-1}) \leq 2\l( \f{ 32 \hat L(\Wt 0)}{\gamma^2 \snorm{\mu}^2 \alpha T} \r)^{1/2}.\]
In particular, for $T \geq 128 \hat L(\Wt 0)/\left(  \gamma^{2}\snorm{\mu}^2 \alpha \eps^{2}\right)$, we have $\hat G(\Wt {T-1})\leq \eps/2$.
\end{restatable}

We now have all the results necessary to prove our main theorem.

\begin{proof}[Proof of Theorem~\ref{theorem:main}]
By Lemma~\ref{lemma:sample.facts} and Lemma~\ref{lemma:initialization.norm}, a `good run' occurs with probability at least $1-2\delta$. If $n \snorm \mu^4 < Cp \log 2$ then the generalization bound holds trivially, hence in the remainder of the proof we will assume $n\snorm \mu^4 \geq Cp \log (2)$.  Since a good run occurs, by taking $C$ to be a large enough constant we can apply Lemma~\ref{lemma:normalized.margin.lb}.  Using this as well as Lemma~\ref{lemma:lipschitz.concentration}, we have with probability at least $1-2\delta$,
\begin{align*}
 &\qquad \P_{(x,y)\sim \pnoise} \big(y \neq \sgn(f(x; W))\big) \\
 &\leq \eta + \exp \l( - c \l( \f{ 0 \vee \E[f(\mu+z;W)]}{\snorm{W}_2}\r)^2 \r) + \exp \l( - c \l( \f{ 0 \vee \E[-f(-\mu+z;W)]}{\snorm{W}_2}\r)^2 \r)\\
 &\leq \eta + 2 \exp \l( - c \l( \f{\gamma^4 n\snorm{\mu}^4 }{32^2\max(C_1, C_2^2) p} \r) \r).
\end{align*} 
By taking $C$ to be a large enough constant we have that $T \geq 128 \hat L(\Wt 0) /\left(\gamma^2 \snorm{\mu}^2 \alpha \eps^{2}\right)$, and hence by Lemma~\ref{lemma:optimization.guarantee}, we have
\[ \hat G(\Wt {T-1}) \leq \eps/2.\]
Since $\eps < 1/(2n)$ and $g(z) = -\ell'(z) < 1/2$ if and only if $z >0$, we know that $y_i f(x_i; \Wt {T- 1}) >0$ for every $i\in [n]$.  We are working with the logistic loss, and hence we have $\f 12 \ell(y_i f(x_i; \Wt {T-1}))\leq g(y_i f(x_i; \Wt {T-1}))$ for every $i\in [n]$, which implies that
\[ \hat L(\Wt {T-1}) = \f 1 n \summ i n \ell(y_i f(x_i; \Wt {T-1})) \leq \f 2 n \summ i n -\ell'(y_i f(x_i; \Wt {T-1})) = 2 \hat G(\Wt {T-1}) \leq \eps.\]
\end{proof}

\section{Discussion}
We have shown that neural networks trained by gradient descent will interpolate noisy training data and still generalize close to the noise rate when the data comes from a mixture of well-separated sub-Gaussian distributions and the dimension of the data is larger than the sample size.  Our results mimic those established by~\citet{chatterji2020linearnoise} for linear classifiers, but they hold for the much richer class of two-layer neural networks.  

Our proof technique relies heavily upon the assumption that the number of samples is much less than the ambient dimension.  This assumption allows for every pair of distinct samples to be roughly mutually orthogonal so that samples with noisy labels cannot have an outsized effect on the ability for the network to learn a positive margin on clean examples.   Previous work has established a similar `blessing of dimensionality' phenomenon:~\citet{belkin2018overfittingperfectfitting} showed that the gap between a particular simplicial interpolation rule and the Bayes error decreases exponentially fast as the ambient dimension increases, mimicking the behavior we show in Theorem~\ref{theorem:main}.   In the linear regression setting, it is known that for the minimum norm solution to generalize well it is necessary for the dimension of the data $p$ to be much larger than $n$~\citep{bartlett2020.benignoverfitting.pnas}. It has also been shown that if the ambient dimension is one, local interpolation rules necessarily have suboptimal performance~\citep{ji2021earlystopped}. Taken together, these results suggest that working in high dimensions makes it easier for benign overfitting to hold, but it is an interesting open question to understand the extent to which working in the $p\ge  n$ regime is necessary for benign overfitting with neural networks.  In particular, when can benign overfitting occur in neural networks that have enough parameters to fit the training points $(mp > n)$ but for which the number of samples is larger than the input dimension $(n>p)$?  

In this work we considered a data distribution for which the optimal classifier is linear but analyzed a model and algorithm that are fundamentally nonlinear.  A natural next step is to develop characterizations of benign overfitting for neural networks trained by gradient descent in settings where the optimal classifier is nonlinear.  We believe some of the insights developed in this work may be useful in these settings: in particular, it appears that in the $p\gg n$ setting, the optimization dynamics of gradient descent can become simpler as can be seen by the `loss ratio bound' provided in Lemma~\ref{lemma:loss.ratio}.  On the other hand, we believe the generalization analysis will become significantly more difficult when the optimal classifier is nonlinear.

\section*{Acknowledgements}
We thank Ichiro Hashimoto, Ryota Ushio, Can Le, and Bryn Elesedy for pointing out errors in previous versions of this work, and we thank Christophe Giraud for pointing us to references showing that the test error in Theorem~\ref{theorem:main} is minimax optimal. 
We are thankful for the anonymous reviewers' helpful feedback. 
We gratefully acknowledge the support of the NSF and the Simons Foundation for 
the Collaboration on the Theoretical Foundations of Deep Learning through awards DMS-2023505, DMS-2031883, 
and \#814639.

\appendix
\newpage
{\hypersetup{linkcolor=Black}\tableofcontents}
\section{Omitted Proofs from Section~\ref{sec:thmproof}}\label{app:theorem_proof}
In this section we provide a proof of all of the lemmas presented in Section~\ref{sec:thmproof}.   We remind the reader that throughout this section, we assume that Assumptions~\ref{a:samples} through~\ref{a:sinit} are in force.  

First in Section~\ref{ss:concentration_results} we prove the concentration results, Lemmas~\ref{lemma:lipschitz.concentration} and \ref{lemma:initialization.norm}. Next, in Section~\ref{ss:structural_results} we prove the structural results, Lemmas~\ref{lemma:nn.smooth},~\ref{lemma:loss.smooth} and~\ref{lemma:nn.grad.ip.identity}. In Section~\ref{ss:margin_increase_lemma} we prove Lemma~\ref{lemma:margin.test.sample} that demonstrates that the margin on a test point increases with training. In Section~\ref{ss:loss_ratio_bound} we prove Lemma~\ref{lemma:loss.ratio} that guarantees that the ratio of the surrogate losses remains bounded throughout training, while in Section~\ref{ss:refined_upper_bound} we prove Lemma~\ref{l:refined_grad_norm_upper_bound} that bounds the growth of the norm of the parameters. Next, in Section~\ref{ss:normalized_margin_lb} we prove Lemma~\ref{lemma:normalized.margin.lb} that provides a lower bound on the normalized margin on a test point. Finally, in Section~\ref{ss:optimization_guarantee}, we prove Lemma~\ref{lemma:optimization.guarantee} that is useful in proving that the training error and loss converge to zero.

\subsection{Concentration Inequalities}\label{ss:concentration_results}
In this subsection we prove the concentration results Lemmas~\ref{lemma:lipschitz.concentration} and \ref{lemma:initialization.norm}. 
\subsubsection{Proof of Lemma~\ref{lemma:lipschitz.concentration}}
Let us restate Lemma~\ref{lemma:lipschitz.concentration}.
\marginlowerfirst*
\begin{proof}

First, note that if either $\E[f(\mu+z;W)] \leq 0$ or $\E[-f(-\mu+z;W)] \leq 0$, then the inequality holds trivially.  Thus it suffices to consider the case where both of these quantities are non-negative.   

By definition, clean data $x$ is sampled by first sampling a clean label $\tilde y \sim \Unif(\{\pm 1 \})$, then sampling $z\sim \pclust$ and setting $x=\tilde y \mu + z$.  The observed label $y$ is then sampled such that $\P(y \neq \tilde y) \leq \eta$, and so following~\citet[Lemma 9]{chatterji2020linearnoise} it holds that
\[\P( y f(x;W) < 0) \leq \eta + \P(\tilde y f(x;W) < 0) . \]
It therefore suffices to bound the second quantity appearing on the right-hand side above.  Using  $x=\tilde y \mu + z$ and $\tilde y \sim \Unif(\{\pm 1 \})$ we get
\begin{align*}
    \P(\tilde y f(x;W) < 0) &= \P(\tilde y = 1) \P(f(x; W) < 0 | \tilde y = 1) + \P(\tilde y = -1) \P(-f(x;W) < 0) | \tilde y = -1) \\
    &= \f 12 \l[ \P_{z\sim\pclust}(f(\mu + z; W) < 0) + \P_{z\sim \pclust}(-f(-\mu+z; W) < 0) \r].
\end{align*}
It therefore suffices to show that $ \P(f(\mu + z; W) < 0)$ and $ \P(-f(-\mu + z; W) < 0)$ are small.  Towards this end, we first note that $f$ is a $\snorm{W}_2$-Lipschitz function of the input $x$: let $x,x'\in \R^p$, and consider
\begin{align*}
|f(x; W) - f(x'; W)| &= \l|\summ j m a_j [\phi(\sip{w_j}{x}) - \phi(\sip{w_j}{x'})]\r| \\
&\overset{(i)}\leq \summ j m |a_j| |\sip{w_j}{x-x'}| \\
&\overset{(ii)}\leq \sqrt{\summ j m a_j^2} \sqrt{\summ j m \sip{w_j}{x-x'}^2}\\
&= \snorm{W(x-x')} \\
&\overset{(iii)}\leq \snorm{W}_2 \snorm{x-x'}.
\end{align*}
Above, $(i)$ uses that $\phi$ is 1-Lipschitz, and $(ii)$ follows by the Cauchy--Schwarz inequality.  Inequality $(iii)$ is by the definition of the spectral norm.  This shows that $f(\cdot; W)$ is $\snorm{W}_2$-Lipschitz.

Continuing, we have,
\begin{align*} \numberthis \label{eq:test.error.prelim}
    \P(f(\mu+z;W) < 0) &= \P(f(\mu+z; W) - \E[f(\mu+z;W)] < -\E[f(\mu+z;W)]).
\end{align*}
Now, the mapping $z\mapsto z + \mu$ is 1-Lipschitz, and since $x\mapsto f(x;W)$ is $\snorm{W}_2$-Lipschitz, we have that $z\mapsto f(\mu+z; W)$ is also $\snorm{W}_2$-Lipschitz.  Since $\pclust$ is $\lambda$-strongly log-concave,  by~\citet[Theorem 2.7 and Proposition 1.10]{ledoux2001concentration}, since $z \mapsto f(\mu+z;W)$ is $\snorm{W}_2$-Lipschitz, there is an absolute constant $c>0$ such that for any $q\geq 1$, $\snorm{\tilde y f(x; W) - \E [\tilde y f(x; W)]}_{L^q} \leq c \snorm{W}_2 \sqrt{q/\lambda}$.  This behavior of the growth of $L^q$ norms is equivalent to $\tilde y f(x; W) - \E [\tilde y f(x; W)]$ having sub-Gaussian norm $c' \snorm{W}_2/\sqrt{\lambda}$ for some absolute constant $c'>0$, by~\citet[Proposition 2.5.2]{vershynin}.  Thus, there is an absolute constant $c''>0$ such that for any $t\geq 0$,
\begin{align*}
    \P(|f(\mu+z;W) - \E[f(\mu+z;W)] | \geq t) \leq 2 \exp \l( - c'' \l( \f{t}{\snorm W_2}\r)^2\r).
\end{align*}
Likewise,
\begin{align*}
    \P(|-f(-\mu+z;W) - \E[-f(-\mu+z;W)] | \geq t) \leq 2 \exp \l( - c'' \l( \f{t}{\snorm W_2}\r)^2\r).
\end{align*}

Thus taking $t= \E[f(\mu+z;W)]\geq 0$ for the first inequality and $t=\E[-f(-\mu+z;W)]\geq 0$ for the second (and using that we have assumed both of these are non-negative, per the comment at the beginning of the proof of this lemma), using~\eqref{eq:test.error.prelim} we get
\begin{align*}
    \P( \tilde y f(x;W) < 0) &\leq \exp \l( - c'' \l( \f{ 0\vee \E[f(\mu+z;W)]}{\snorm{W}_2}\r)^2 \r) + \exp \l( - c \l( \f{ 0\vee \E[-f(-\mu+z;W)]}{\snorm{W}_2}\r)^2 \r).
\end{align*}
\end{proof}

\subsubsection{Proof of Lemma~\ref{lemma:initialization.norm}}
Now let us restate and prove Lemma~\ref{lemma:initialization.norm}. 
\normbound*
\begin{proof}
Note that $\snorm{\Wt 0}_F^2$ is a $\sinit^2$-multiple of a chi-squared random variable with $mp$ degrees of freedom.  By concentration of the $\chi^2$ distribution \cite[Example~2.11]{wainwright}, for any $t\in (0,1)$,
\[ \P \l( \l| \f 1 {mp\sinit^2} \snorm{\Wt 0}_F^2 - 1\r| \geq t\r) \leq 2 \exp(-mp t^2/8).\]
In particular, if we choose $t = \sqrt{8 \log(4/\delta)/md}$ and use Assumption~\ref{a:dimension}, we get that $t\leq 1/2$ and so with probability at least $1-\delta/2$, we have
\[ \snorm{\Wt 0}_F^2 \leq \f 32 mp \sinit^2.\]
As for the spectral norm, since the entries of $ \Wt 0/\sinit$ are i.i.d. standard normal random variables, by~\citet[Theorem 4.4.5]{vershynin} there exists a universal constant $c>0$ such that for any $u \geq 0$, with probability at least $1-2\exp(-u^2)$, we have
\[ \snorm{\Wt 0}_2 \leq c \sinit ( \sqrt m + \sqrt p + u).\]
In particular, taking $u = \sqrt{\log(4/\delta)}$ we have with probability at least $1-\delta/2$, $\snorm{\Wt 0}_2\leq c \sinit (\sqrt m + \sqrt p + \sqrt{\log(4/\delta)}$.  Since $\sqrt p \geq \sqrt{\log(4/\delta)}$ by Assumption~\ref{a:dimension}, the proof is completed by a union bound over the claims on the spectral norm and the Frobenius norm. 
\end{proof}

\subsection{Structural Results}\label{ss:structural_results}
As stated above in this section we prove Lemmas~\ref{lemma:nn.smooth},~\ref{lemma:loss.smooth} and \ref{lemma:nn.grad.ip.identity}.

\subsubsection{Proof of Lemma~\ref{lemma:nn.smooth}}
We begin by restating and proving Lemma~\ref{lemma:nn.smooth}.
\smoothness*
\begin{proof}
Since $\phi$ is twice differentiable, $\phi'$ is continuous and so by Taylor's theorem, for each $j\in [m]$, there exist constants $t_j = t_j(w_j, v_j, x)\in \R$,
\begin{equation*}
    \phi(\sip{w_j}x) - \phi(\sip{v_j}x) = \phi'(\sip{v_j}x) \cdot \sip{w_j-v_j}x + \frac{\phi''(t_j)}{2} (\sip{w_j-v_j}x)^2,
\end{equation*}
where $t_j$ lies between $\sip{w_j}x$ and $\sip{v_j}x$.  We therefore have
\begin{align*}
    f(x; W) - f(x; V ) &=  \summ j ma_j [ \phi(\sip{w_j}x) - \phi(\sip{v_j}x)] \\
    &=  \summ j m a_j \l[\phi'(\sip{v_j}x) \cdot \sip{w_j-v_j}x +\frac{\phi''(t_j)}2 \sip{w_j-v_j}x^2 \r]\\
    &= \sip{\nabla f(x; V )}{W-V} +   \summ j m a_j \frac{\phi''(t_j)}{2} \sip{w_j-v_j}x^2.
\end{align*}
The last equality follows since we can write
\begin{equation} \label{eq:gradient.dx.formula}
\nabla f(x; V ) =   D_x^V a x^\top,\quad \text{where }\quad  D_x^V := \diag(\phi'(\sip{v_j}x)),
\end{equation}
and thus
\[ \sip{\nabla f(x; V )}{W-V} =  \tr(x a^\top D_x^V (W-V)) =  a^\top D_x^V(W-V) x =  \sum_j a_j \phi'(\sip{v_j}x) \sip{w_j-v_j}x.\]
For the final term, we have
\begin{align*}
    \l| \summ j m a_j \frac{\phi''(\xi_j)}{2} \sip{w_j-v_j}x^2 \r| &\leq  \summ j m |a_j| \frac{|\phi''(t_j)|}{2} \sip{w_j-v_j}x^2\\
    &\leq \frac{ H}{2 \sqrt m} \summ j m \sip{w_j-v_j}x^2 \\ 
    &= \frac { H} {2 \sqrt m} \pnorm{(W-V)x}2^2 \\
    &\leq \frac { H} {2\sqrt m} \pnorm{W-V}2^2 \snorm{x}_2^2.
\end{align*}
\end{proof}
\subsubsection{Proof of Lemma~\ref{lemma:loss.smooth}}
Next we prove Lemma~\ref{lemma:loss.smooth} that establishes that the loss is smooth.
\losssmoothness*
\begin{proof}
Since a good run occurs, all the events in Lemma~\ref{lemma:sample.facts} hold.  We thus have
\begin{align*}
    \norm{\nabla \hat L(W)}_F &= \norm{\f 1 n \summ in g(y_i f(x_i; W)) y_i \nabla f(x_i; W)}_F\\
    &\overset{(i)}\leq \f {1} n \summ i n g(y_i f(x_i; W)) \pnorm{\nabla f(x_i; W)}F \\
    &\overset{(ii)}\leq \f { \sqrt{C_1 p}} n \summ i n g(y_i f(x_i; W)) = \sqrt{C_1 p} \hat G(W) \\
    &\overset{(iii)}\leq \f { \sqrt{C_1 p}}  n\summ i n \min(\ell(y_i f(x_i; W)), 1) \\
    &\overset{(iv)}{\le} \sqrt{C_1 p }(\hat L(W) \wedge 1).
\end{align*}
In $(i)$ we have used Jensen's inequality.  In $(ii)$ we have used that $\phi$ is 1-Lipschitz so that $\pnorm{\nabla f(x_i; W)}F^2 = \pnorm{D_i^W a x_i^\top}F^2 = \norm{D_i^W a}_2^2 \norm{x_i}_2^2 \leq C_1 p$ by Event~\eqref{eq:xk.norm}, where $D_i^W = D_{x_i}^W$ is defined in Equation~\eqref{eq:gradient.dx.formula}.  In $(iii)$ we use that $0 \leq g(z) \leq 1 \wedge \ell(z)$. In $(iv)$ we use Jensen's inequality since $z\mapsto \min\left\{z,1\right\}$ is a concave function.

Next we show that the loss has Lipschitz gradients.  First, we have the decomposition
\begin{align*}
\snorm{\nabla \hat L(W) - \nabla \hat L(V)}_F &= \norm{\f 1 n \summ i n \l[ g(y_i f(x_i; W)) y_i \nabla f(x_i; W) - g(y_i f(x_i; V))y_i \nabla f(x_i; V) \r] }_F \\
&\leq \f 1n \summ i n \snorm{\nabla f(x_i; W)}_F |g(y_i f(x_i; W))- g(y_i f(x_i; V))| \\
&\qquad + \f 1 n \summ i n \snorm{\nabla f(x_i; W) - \nabla f(x_i; V)}_F\\
&\overset{(i)}\leq \f 1n \summ i n \snorm{\nabla f(x_i; W)}_F |f(x_i; W)- f(x_i; V)| \\
&\qquad + \f 1 n \summ i n \snorm{\nabla f(x_i; W) - \nabla f(x_i; V)}_F.\numberthis \label{eq:loss.gradient.lipschitz.intermediate}
\end{align*} 
In $(i)$, we use that $g=-\ell'$ (the negative derivative of the logistic loss) is 1-Lipschitz.  Therefore, to show that the loss has Lipschitz gradients, it suffices to show that both the network and the gradient of the network are Lipschitz with respect to the first layer weights.   We first show that the network is Lipschitz with respect to the network parameters:
\begin{align*}
|f(x; W) - f(x; V)|^2 &= \l| \summ j m a_j [\phi(\sip{w_j}{x}) - \phi(\sip{v_j}{x})] \r|^2 \\
&\leq \l( \summ j m a_j^2 \r) \cdot \summ j m |\phi(\sip{w_j}{x}) - \phi(\sip{v_j}{x})|^2 \\
&\leq \summ j m |\sip{w_j}{x} - \sip{u_j}{x}|^2 \\
&= \snorm{(W-V)x}^2 \\
&\leq \snorm{x}^2 \snorm{W-V}_2^2.\numberthis \label{eq:network.lipschitz}
\end{align*}
As for the gradients of the network, again recalling the $D_x^W$ notation from Equation~\eqref{eq:gradient.dx.formula}, we have
\begin{align*}
\snorm{\nabla f(x; W) - \nabla f(x; V)}_F^2 &= \snorm{(D_x^W- D_x^V) a x^T}^2 \\
&\leq \snorm{x}^2 \snorm{(D_x^W - D_x^V)a}^2 \\
&=\snorm{x}^2  \summ j m a_j^2 [\phi'(\sip{w_j}{x}) - \phi'(\sip{v_j}{x})]^2 \\
&\leq \snorm{x}^2 \cdot \f {H^2}{m} \summ j m |\sip{w_j}{x} - \sip{v_j}{x}|^2 \\
&= H^2 \snorm{x}^2 \cdot \f 1 m \snorm{(W-V)x}^2\\
&\leq \f{H^2}{m} \snorm{x}^4 \snorm{W-V}_2^2.\numberthis \label{eq:network.gradient.lipschitz}
\end{align*}
Continuing from~\eqref{eq:loss.gradient.lipschitz.intermediate}, we have
\begin{align*}
\snorm{\nabla \hat L(W) - \nabla \hat L(V)}_F &\leq  \f 1n \summ i n \snorm{\nabla f(x_i; W)}_F |f(x_i; W)- f(x_i; V)| \\
&\qquad + \f 1 n \summ i n \snorm{\nabla f(x_i; W) - \nabla f(x_i; V)}_F\\
&\overset{(i)} \leq \sqrt{C_1 p} \cdot \f 1 n \summ i n |f(x_i; W) - f(x_i; V)| + \f{C_1 Hp}{\sqrt m} \snorm{W-V}_2 \\
&\overset{(ii)}\leq  C_1p \left (  1 + \f{ H}{\sqrt m}  \right ) \snorm{W-V}_2.\numberthis \label{eq:logistic.lipschitz.gradients}
\end{align*}
In $(i)$ we use that $\phi$ being 1-Lipschitz implies $\snorm{\nabla f(x_i; W)}_F = \snorm{x_i} \snorm{D_i^W a}\leq \sqrt {C_1p}$ for the first term, and~\eqref{eq:network.gradient.lipschitz} together with~\eqref{eq:xk.norm}.  In $(ii)$, we use~\eqref{eq:network.lipschitz} and~\eqref{eq:xk.norm}. 
\end{proof}
\subsubsection{Proof of Lemma~\ref{lemma:nn.grad.ip.identity}}
Finally, we prove Lemma~\ref{lemma:nn.grad.ip.identity} that bounds the correlation between the gradients.
\gradientcorrelation*
\begin{proof}
Recall the notation $D_i^W := \diag(\phi'(\sip{w_j}{x_i})\in \R^{m\times m}$.  By definition,
\begin{align*}
\sip{\nabla f(x_i; W)}{\nabla f(x_k; W) }&= \tr(x_i a^\top D_i^{W} D_{k}^{W} a x_k^\top) \\
    &= \tr\l(x_i^\top x_k a^\top D_i^{W} D_{k}^{W} a\r) \\
 	&= \sip{x_i}{x_k} a^\top D_i^W D_k^W a \\
 	&= \sip{x_i}{x_k} \summ j m a_j^2 \phi'(\sip{w_j}{x_i}) \phi'(\sip{w_j}{x_k}) \\
 	&= \sip{x_i}{x_k}  \cdot \f 1m \summ j m \phi'(\sip{w_j}{x_i}) \phi'(\sip{w_j}{x_k}).\numberthis \label{eq:gradient.ip.identity}
\end{align*}
Since a good run occurs, all the events in Lemma~\ref{lemma:sample.facts} hold.  We can therefore bound,
\begin{align*}
|\sip{\nabla f(x_i; W)}{\nabla f(x_k; W) }| \overset{(i)}\leq |\sip{x_i}{x_k}| 
\overset{(ii)}\leq C_1 \l( \snorm{\mu}^2 + \sqrt{p\log(n/\delta)}\r).
\end{align*}
Inequality $(i)$ uses that $|\phi'(z)| \leq 1$, while inequality $(ii)$ uses Event~\eqref{eq:xi.xk.innerproduct} from Lemma~\ref{lemma:sample.facts}. This completes the proof for part (a).  For part (b), we continue from~\eqref{eq:gradient.ip.identity} to get
\begin{align*}
\snorm{\nabla f(x_i; W)}_F^2 &= \snorm{x_i}^2 \cdot \f 1 m \summ j m \phi'(\sip{w_j}{x_i})^2.
\end{align*}
By the assumption on $\phi$, we know $\phi'(z)\geq \gamma > 0$ for every $t\in \R$.  Now we can use Lemma~\ref{lemma:sample.facts}, which states that $p/C_1 \leq \snorm{x_i}^2 \leq C_1p$ for all $i$.   In particular, we have
\begin{align*}
\f {p}{C_1 } \cdot \gamma^2  &\leq \snorm{x_i}^2 \cdot \f 1 m \summ j m \phi'(\sip{w_j}{x_i})^2  = \snorm{\nabla f(x_i; W)}_F^2 \leq C_1p.
\end{align*}
\end{proof}

\subsection{Proof of Lemma~\ref{lemma:loss.ratio}}\label{ss:loss_ratio_bound}

Let us first restate the lemma.
\lossratio*

Lemma~\ref{lemma:loss.ratio} is a bound on the maximum possible ratio of \emph{sigmoid losses} ($z\mapsto 1/(1+\exp(z))$). We find that the ratio of the sigmoid losses is closely related to the ratio of exponential losses, and that the ratio of exponentials is particularly well-behaved.  
In what follows, we shall show that the ratio of the \textit{exponential} losses is bounded throughout training and that this implies that the ratio of the sigmoid losses is also bounded.  This sets up the following proof roadmap:
\begin{enumerate}
\item We first characterize how the ratio of exponential losses increases from one iteration to the next.
\item We characterize how ratios of exponential losses relate to ratios of sigmoid losses.
\item We show that gradient descent quickly enters a regime where we can essentially treat the exponential and sigmoid losses interchangeably.
\item We then argue inductively to show that the exponential loss ratio (and thus the sigmoid loss ratio) can never be too large. 
\end{enumerate}

The following lemma addresses the first step above.  It provides a bound on the ratio of the exponential losses at time $t+1$ in terms of the ratio of the exponential losses at time $t$ for any two samples $(x_i, y_i)$ and $(x_j, y_j)$.  The lemma shows that if the ratio of the sigmoid losses $\lpit it/\lpit jt$ is large, 
and if the step size is relatively small, then we can show that the ratio of the exponential losses decreases at the following iteration. 

\begin{lemma}\label{lemma:exp.loss.ratio}
On a good run, for $C>1$ sufficiently large, we have for all $i,j\in [n]$ and $t\geq 0$, 
\begin{align*}
\frac{ \exp\big (-y_i f(x_i;\Wt {t+1})\big) }{ \exp\big(-y_j f(x_j;\Wt {t+1}) \big)} &\leq \frac{\exp\big (-y_i f(x_i;\Wt t) \big)}{ \exp\big(-y_j f(x_j;\Wt t)\big)}\\
&\qquad \times \exp \l( - \f{ \lpit jt \alpha \gamma^2 p}{C_1 n} \l( \f{ \lpit it}{\lpit jt} - \f{ C_1^2}{\gamma^{2}}\r) \r) \\
&\qquad \times \exp\l(2 C_1 \alpha \l( \snorm{\mu}^2 + 2 \sqrt{p\log(n/\delta)} \r) \hat G(\Wt t)\r).
\end{align*}

\end{lemma}
\begin{proof} 
Without loss of generality, it suffices to consider how the exponential loss ratio of the first sample to the second sample changes.  To this end, let us denote
\[ A_t := \f{ \exp(-y_1 f(x_1; \Wt t))}{\exp(-y_2 f(x_2; \Wt t))}.\]
We now calculate the exponential loss ratio between two samples at time $t+1$ in terms of the exponential loss ratio at time $t$.   Recalling the notation $\lpit it := g(y_i f(x_i; \Wt t))$, we can calculate, 
\begin{align*}
    A_{t+1} &= \frac{\exp(-y_1f(x_1; W^{(t+1)}))}{\exp(-y_2f(x_2; W^{(t+1)}))} \\
    &= \frac{\exp\left(-y_1f_1\left(W^{(t)}-\alpha \nabla \hat{L}(W^{(t)})\right)\right)}{\exp\left(-y_2f_2\left(W^{(t)}-\alpha \nabla \hat{L}(W^{(t)})\right)\right)} \\
    &\overset{(i)}\le  \frac{\exp\left(-y_1f\left(x_1; W^{(t)}\right)+y_1\alpha  \left\langle \nabla f(x_1; W^{(t)}),\nabla \hat{L}(W^{(t)})\right\rangle \right)}{\exp\left(-y_2f\left(x_2; W^{(t)}\right)+y_2\alpha \left\langle \nabla f(x_2; W^{(t)}),\nabla \hat{L}(W^{(t)})\right\rangle \right)} \exp\left(\frac{ H C_1 p \alpha^2}{\sqrt{m}}\lVert\nabla \hat{L}(W^{(t)})\rVert^2\right) \\
    &\overset{(ii)}=  A_t\cdot\frac{\exp\left(y_1\alpha  \left\langle \nabla f(x_1; W^{(t)}),\nabla \hat{L}(W^{(t)})\right\rangle \right)}{\exp\left(y_2\alpha \left\langle \nabla f(x_2; W^{(t)}),\nabla \hat{L}(W^{(t)})\right\rangle \right)} \exp\left(\frac{ H C_1 p\alpha^2}{\sqrt{m}}\lVert\nabla \hat{L}(W^{(t)})\rVert^2\right) \\
    &= A_t\cdot\frac{\exp\left( -\frac{\alpha \ }{n}\sum_{k=1}^n y_1 y_k \lpit kt \sip{\nabla f(x_1; \Wt t)}{\nabla f(x_k; \Wt t)} \right)}{\exp\left( -\frac{ \alpha  }{n}\sum_{k=1}^n y_2 y_k\lpit kt \sip{\nabla f(x_2; \Wt t)}{\nabla f(x_k; \Wt t)} \right)} \exp\left(\frac{ H C_1 p\alpha^2}{\sqrt{m}}\lVert\nabla \hat{L}(W^{(t)})\rVert^2\right) \\
    &= A_t\cdot\exp\left(-\frac{ \alpha}{n}\left(\lpit 1t \snorm{\nabla f(x_1; \Wt t)}_F^2 - \lpit 2t \snorm{\nabla f(x_2; \Wt t)}_F^2  \right)\right)   \\
    &\qquad \times \frac{\exp\left( -\frac{ \alpha }{n}\sum_{k>1} y_1 y_k \lpit kt\sip{\nabla f(x_1; \Wt t)}{\nabla f(x_k; \Wt t)} \right)}{\exp\left( -\frac{ \alpha }{n}\sum_{k\neq 2} y_2 y_k \lpit kt\sip{\nabla f(x_2; \Wt t)}{\nabla f(x_k; \Wt t)} \rangle \right)}  \\
    &\qquad \quad \times \exp\left(\frac{ H C_1 p\alpha^2}{\sqrt{m}}\lVert\nabla \hat{L}(W^{(t)})\rVert^2\right).\numberthis \label{eq:loss.ratio.init.equality}
    \end{align*} 
Inequality $(i)$ uses Lemma~\ref{lemma:nn.smooth} and Event~\eqref{eq:xk.norm} which ensures that $\lv x_i \rv^2\le C_1 p$, and $(ii)$ uses that $A_t$ is the ratio of the exponential losses.  We now proceed to bound each of the three terms in the product separately.  For the first term, by Part~(b) of Lemma~\ref{lemma:nn.grad.ip.identity}, we have for any $i\in [n]$,
\begin{equation} 
\f{\gamma^2 p }{C_1}  \leq \snorm{\nabla f(x_i; \Wt t)}_F^2 \leq C_1p.
\label{eq:xi.relu.nac.consequence}
\end{equation}
Therefore, we have
\begin{align*}
&\exp\left(-\frac{ \alpha}{n}\left(\lpit 1t \snorm{\nabla f(x_1; \Wt t)}_F^2 - \lpit 2t \snorm{\nabla f(x_2; \Wt t)}_F^2  \right)\right) \\
&\qquad = \exp\left(-\frac{ \lpit 2t \alpha}{n}\left(\f{ \lpit 1t}{\lpit 2t} \snorm{\nabla f(x_1; \Wt t)}_F^2 - \snorm{\nabla f(x_2; \Wt t)}_F^2 \right)\right)\\
&\qquad \overset{(i)}\leq \exp\left(-\frac{ \lpit 2t \alpha}{n}\left(\f{ \lpit 1t}{\lpit 2t} \cdot \f{\gamma^2 p}{C_1}  -  C_1p \right)\right)\\
&\qquad = \exp\l(-\frac{ \lpit 2t \alpha \gamma^2 p  }{C_1  n}\l (\f{ \lpit 1t}{\lpit 2t} - \frac{C_1^2}{\gamma^2}  \r)\r).\numberthis \label{eq:loss.ratio.firstterm}
\end{align*}
Inequality $(i)$ uses~\eqref{eq:xi.relu.nac.consequence}.  This bounds the first term in~\eqref{eq:loss.ratio.init.equality}.  

For the second term, we again use Lemma~\ref{lemma:nn.grad.ip.identity}: we have for any $i \neq k$,
\begin{equation} |\sip{\nabla f(x_i; W)}{\nabla f(x_k; W)}| \leq C_1 \l( \snorm{\mu}^2 + \sqrt{p \log(n/\delta)} \r).\label{eq:xi.relu.nac.consequence.2}
\end{equation}
This allows for us to bound,
\begin{align*}
& \frac{\exp\left( -\frac{ \alpha }{n}\sum_{k>1} y_1 y_k\lpit kt\sip{\nabla f(x_1; \Wt t)}{\nabla f(x_k; \Wt t)} \right)}{\exp\left( -\frac{ \alpha }{n}\sum_{k\neq 2} y_2 y_k \lpit kt\sip{\nabla f(x_2; \Wt t)}{\nabla f(x_k; \Wt t)} \rangle \right)}\\
&\overset{(i)}\leq \exp\l( \f \alpha n \sum_{k \neq 1} \lpit kt |\sip{\nabla f(x_1; \Wt t)}{\nabla f(x_k; \Wt t)}| + \f \alpha n \sum_{k\neq 2} \lpit kt |\sip{\nabla f(x_2; \Wt t)}{\nabla f(x_k; \Wt t)}|  \r) \\
&\overset{(ii)}\leq \exp\l( \f \alpha n \sum_{k \neq 1} \lpit kt \cdot C_1 \l( \snorm{\mu}^2 + \sqrt{p\log(n/\delta)}\r) + \f \alpha n \sum_{k\neq 2} \lpit kt \cdot C_1 \l( \snorm{\mu}^2 + \sqrt{p\log(n/\delta)}\r)  \r) \\
&\overset{(iii)}\leq \exp\l( 2 \f \alpha n \summ k n  \lpit kt \cdot C_1 \l( \snorm{\mu}^2 + \sqrt{p\log(n/\delta)}\r) \r)\\
&= \exp\l( 2C_1\alpha \l(\snorm{\mu}^2 + \sqrt{p \log(n/\delta)}\r) \hat G(\Wt t)\r).\numberthis \label{eq:loss.ratio.secondterm}
\end{align*}
Inequality $(i)$ uses the triangle inequality.  Inequality $(ii)$ uses that $\lpit kt\geq 0$ for all $k\in [n]$ and eq.~\eqref{eq:xi.relu.nac.consequence.2}.  Inequality $(iii)$ again uses that $\lpit kt\geq 0$.  

Finally, for the third term of~\eqref{eq:loss.ratio.init.equality}, we have
\begin{align*}
\exp\left(\frac{ H C_1 p\alpha^2}{\sqrt{m}}\lVert\nabla \hat{L}(W^{(t)})\rVert^2\right) \overset{(i)}\leq \exp \l( \f{ H C_1^2 p^2 \alpha^2}{\sqrt m} \hat G(\Wt t)\r) \overset{(ii)}\leq \exp\l( \alpha \sqrt{p} \hat G(\Wt t)\r).\numberthis \label{eq:loss.ratio.thirdterm}
\end{align*}
Inequality $(i)$ uses Lemma~\ref{lemma:loss.smooth}, while $(ii)$ uses that for $C>1$ sufficiently large, by Assumption~\ref{a:stepsize} we have $H C_1^2 p^2 \alpha/\sqrt m \leq \sqrt p$. Putting~\eqref{eq:loss.ratio.firstterm},~\eqref{eq:loss.ratio.secondterm} and~\eqref{eq:loss.ratio.thirdterm} into~\eqref{eq:loss.ratio.init.equality}, we get
\begin{align*}
A_{t+1} &\leq A_t \cdot \exp\l(-\frac{ \lpit 2t \alpha \gamma^2 p}{C_1  n}\l (\f{ \lpit 1t}{\lpit 2t} - \frac{C_1^2}{\gamma^2}  \r)\r)   \\
&\qquad \quad \times \exp\l( 2C_1\alpha \l(\snorm{\mu}^2 + \sqrt{p \log(n/\delta)}\r) \hat G(\Wt t)\r)\cdot \exp\l( \alpha \sqrt{p} \hat G(\Wt t)\r)\\
&\le A_t \cdot  \exp\l(-\frac{ \lpit 2t \alpha \gamma^2 p}{C_1  n}\l (\f{ \lpit 1t}{\lpit 2t} - \frac{C_1^2}{ \gamma^{2}} \r)\r) \\ &\qquad \quad \times  \exp\l( 2C_1\alpha \l(\snorm{\mu}^2 + 2\sqrt{p \log(n/\delta)}\r) \hat G(\Wt t)\r).\numberthis\label{eq:loss.ratio.intermediate}
\end{align*}  
This completes the proof.
\end{proof}

Lemma~\ref{lemma:exp.loss.ratio} shows at a high-level that the ratio of the exponential losses can decrease if the ratio of the sigmoid losses is large and the step-size is small.  We therefore need to characterize how the ratio of the exponential losses relates to the ratio of the sigmoid losses.  We do so in the following fact.

\begin{fact}\label{fact:logistic.loss.ratio}
For any $z_1, z_2\in \R$, 
\[ \f{g(z_1)}{g(z_2)} \leq \max\l(2,2  \f{ \exp(-z_1)}{\exp(-z_2)}\r),\]
and if $z_1, z_2 > 0$, then we also have
\[ \f{ \exp(-z_1)}{\exp(-z_2)} \leq 2 \f{ g(z_1)}{g(z_2)}.\]
\end{fact}
\begin{proof}
By definition, $g(z) = -\ell'(z) = 1/(1+\exp(z))$.  Note that $g$ is strictly decreasing, non-negative, and bounded from above by one.  Further, one has the inequalities
\[ \f 12 \exp(-z) \leq g(z) \leq \exp(-z) \quad \text{for all } z \geq 0.\]
We do a case-by-case analysis on the signs of the $z_i$.  
\begin{itemize} 
\item If $z_1\leq 0$ and $z_2\leq 0$, then since $g(z_1)\leq 1$ and $g(z_2) \geq 1/2$, it holds that $g(z_1)/g(z_2) \leq 2$.

\item If $z_1, z_2 \geq 0$, then since $\nicefrac 12 \exp(-z)\leq g(z)\leq \exp(-z)$ we have $g(z_1)/g(z_2)\leq 2 \exp(-z_1)/\exp(-z_2)$.  Similarly, we have $\exp(-z_1)/\exp(-z_2)\leq 2g(z_1)/g(z_2)$.

\item If $z_1\geq 0$ and $z_2\leq 0$, then $g(z_1)/g(z_2) \leq 2$.

\item If $z_1 \leq 0$ and $z_2 \geq 0$, then $g(z_1) / g(z_2) \leq 2/\exp(-z_2) \leq 2 \exp(-z_1)/\exp(-z_2)$.  
\end{itemize} 
This proves the fact. 
\end{proof}

We now begin to prove Lemma~\ref{lemma:loss.ratio}.  We will prove the lemma in a sequence of three steps: 
\begin{enumerate}
    \item First, we will show that a loss ratio bound holds at times $t=0$ and $t=1$.
    \item We shall then show that at time $t=1$, the neural network correctly classifies all training points.
    \item In the final step, we will argue inductively that at every time $t\geq 1$:
    \begin{enumerate}
    \item the network correctly classifies the training points (so that by Fact~\ref{fact:logistic.loss.ratio} the ratio of the losses is approximately the same under both the sigmoid and exponential loss);
    \item the exponential loss can never be too large, since if it is too large then there will be a large ratio of the sigmoid losses $\lpit it/\lpit jt$ which will cause the exponential loss ratio to decrease as a consequence of Lemma~\ref{lemma:exp.loss.ratio}.
\end{enumerate}
\end{enumerate}

To this end, we introduce a final auxiliary lemma that will allow for us to make the argument outlined above. The lemma consists of three parts.   First, that there is a small ratio of the losses at initialization, which addresses the $t=0$ case of Lemma~\ref{lemma:loss.ratio}.  Second, that a sigmoid loss ratio bound implies an increase in the unnormalized margin for all training examples.  Thus, if at any time the network correctly classifies a training point, the network will continue to correctly classify the training point for all subsequent times.  This will allow for us to treat the exponential loss ratio and sigmoid loss ratio equivalently for all times after the first time that we correctly classify all of the training data by Fact~\ref{fact:logistic.loss.ratio}. Lastly, the network correctly classifies all training points after the first step of gradient descent, so that we can use the near-equivalence of the sigmoid and exponential loss ratios after the first step of gradient descent.

\begin{lemma}\label{lemma:nn.interpolates.time1}
\begin{enumerate} 
On a good run, provided $C>1$ is sufficiently large, the following hold.
\item [(a)] An exponential loss ratio bound holds at initialization:
\[ \max_{i,j} \f{ \exp(-y_i f(x_i;\Wt 0))}{\exp(-y_j f(x_j;\Wt 0))} \leq \exp(2).\]
\item [(b)] If there is an absolute constant $C_r>1$ such that if at time $t$ we have $\max_{i,j} \{ \lpit it / \lpit jt \} \leq C_r$, 
then
\[ \text{for all $k\in [n]$},\qquad y_k[f(x_k;\Wt {t+1}) - f(x_k;\Wt t)] \geq \f{ \alpha \gamma^2 p}{4 C_1 C_r n} \hat G(\Wt t),\]
where $C_1$ is the constant from Lemma~\ref{lemma:sample.facts}.
\item [(c)] At time $t=1$ and for all samples $k\in [n]$, we have $y_k f(x_k;\Wt t) > 0$. 
\end{enumerate}

\end{lemma}
\begin{proof}
We shall prove the lemma in parts.

\paragraph{Part (a): loss ratio at initialization.} 
Since $\phi$ is 1-Lipschitz and $\phi(0)=0$, we have by Cauchy--Schwarz,
\[ |f(x; W)| = \l|\summ j m a_j \phi(\sip{w_j}{x}) \r| \leq \sqrt{\summ j m a_j^2} \sqrt{\summ jm \sip{w_j}{x}^2} = \snorm{Wx}_2.\]
Since a good run occurs, all the events in Lemma~\ref{lemma:sample.facts} and Lemma~\ref{lemma:initialization.norm} hold.  In particular, we have $\snorm{\Wt 0}_2\leq C_0\sinit(\sqrt m + \sqrt p)$ and $\snorm{x_i}\leq \sqrt{C_1p}$ for all $i \in [n]$.  We therefore have the bound,
\[ 2 C_0 \sinit \sqrt{C_1 p}(\sqrt m + \sqrt p) \overset{(i)}\leq \f{ 2 C_0 \sqrt{C_1} \alpha\sqrt p (\sqrt m + \sqrt p)}{\sqrt {mp}}  \overset{(ii)}\leq \f{2 C_0 \sqrt{C_1}}{Cp^2} \l( 1 + \sqrt{\f p m }\r) \overset{(iii)}\leq 1,\]
Inequality $(i)$ uses that $\alpha \geq \sinit \sqrt{mp}$ by Assumption~\ref{a:sinit}, inequality $(ii)$ uses that the step-size is small enough by Assumption~\ref{a:stepsize}, and the final inequality $(iii)$ follows by taking $C>1$ large enough.  
We thus have for all $i\in [n]$,
\begin{equation} \label{eq:network.output.bounded.by.one}
|f(x_i; \Wt 0)| \leq \snorm{\Wt 0}_2 \snorm{x_i} \leq 2C_0 \sinit \sqrt {C_1p}(\sqrt m + \sqrt p)\leq 1.
\end{equation}
Thus, 
\begin{equation} \label{eq:exp.loss.ratio.basecase}
\max_{i,j=1,\dots, n} \f{ \exp(- y_i f(x_i;\Wt 0))}{\exp(-y_j f(x_j; \Wt 0))} \leq \exp(2).
\end{equation}

\paragraph{Part (b): margin increase on training points with loss ratio.}

Fix $k\in [n]$.  We now apply Lemma~\ref{lemma:margin.test.sample} to training samples $(x_k, y_k)$: since $\snorm{x_k}^2\leq C_1 p$ for each $k$ on a good run by Lemma~\ref{lemma:sample.facts}, there exist some $\xi_i = \xi(\Wt t, x_i, x_k)\in [\gamma^2, 1]$ such that
\begin{align*}
&y_k[f(x_k;\Wt {t+1}) - f(x_k;\Wt t)] \\
&\quad \overset{(i)}\geq \f \alpha  n \l[ \summ i n \lpit it \xi_i \sip{y_i x_i}{y_k x_k} - \f{H C_1^2 p^2 \alpha}{2\sqrt m} \hat G(\Wt t)\r]\\
&\quad= \f \alpha n \l[ \lpit kt  \xi_k \snorm{x_k}^2  + \sum_{i\neq k} \lpit it \xi_i \sip{y_i x_i}{y_k x_k} - \f{ H C_1^2 p^2 \alpha}{2 \sqrt m} \hat G(\Wt t)  \r]\\
&\quad \overset{(ii)}\geq \f \alpha n \l[ \lpit kt  \gamma^2 \snorm{x_k}^2  - \max_j \lpit jt \sum_{i\neq k} |\sip{ x_i}{x_k}| - \f{ H C_1^2 p^2 \alpha}{2 \sqrt m} \hat G(\Wt t)  \r] \\
&\quad \overset{(iii)}\geq \f \alpha n \l[ \lpit kt \l( \f{\gamma^2 p}{C_1 } - \f{\max_j \lpit jt}{\lpit kt}\cdot  C_1 n(\snorm{\mu}^2 + \sqrt{p\log(n/\delta)}) \r) - \f{ H C_1^2 p^2 \alpha}{2 \sqrt m} \hat G(\Wt t) \r], 
 \end{align*}
where Inequality $(i)$ uses Lemma~\ref{lemma:margin.test.sample} and that $(1/n)\summ i n \lpit it =\hat G(\Wt t)$. Inequality $(ii)$ uses that $\xi_i \in [\gamma^2, 1]$, while inequality $(iii)$ uses Lemma~\ref{lemma:sample.facts} which gives the bounds $\lv x_k \rv^2 \ge p/C_1$ and $|\langle x_i,x_k\rangle|\le C_1( \lv \mu\rv^2 + \sqrt{p \log(n/\delta)})$.  Continuing we get that
\begin{align*}
&y_k[f(x_k;\Wt {t+1}) - f(x_k;\Wt t)] \\
&\quad \overset{(i)}\geq \f \alpha n \l[ \lpit kt \l( \f{\gamma^2 p}{C_1 } - C_r C_1 n(\snorm{\mu}^2 + \sqrt{p\log(n/\delta)}) \r) - \f{ H C_1^2 p^2 \alpha}{2 \sqrt m} \hat G(\Wt t) \r]\\
&\quad \overset{(ii)}\geq \f \alpha n \l[ \f{ \gamma^2 p}{2 C_1} \lpit kt - \f{ H C_1^2 p^2 \alpha}{2 \sqrt m} \hat G(\Wt t) \r] \\
&\quad \overset{(iii)}\geq \f{ \alpha}{n} \l[ \f{ \gamma^2 p}{2 C_1 C_r } \hat G(\Wt t) -  \f{ H C_1^2 p^2 \alpha}{2 \sqrt m} \hat G(\Wt t) \r]\\
&\quad \overset{(iv)}> \f{ \alpha \gamma^2 p}{4 C_1 C_rn} \hat G(\Wt t). \label{eq:margin.increase.train.points.timet}
\end{align*}
Inequality $(i)$ uses the lemma's assumption that $\max_{i,j} \{\lpit it / \lpit jt\}\leq C_r$.  
Inequality $(ii)$ uses Assumption~\ref{a:dimension} so that $p \gg n \snorm{\mu}^2 \vee \sqrt{p \log(n/\delta)}$.  Inequality $(iii)$ again uses the lemma's assumption of a sigmoid loss ratio bound, so that
\[ \lpit kt = \frac 1 n \summ i n \f{ \lpit it}{\lpit kt} \lpit kt \geq \f{1}{C_r} \frac{1}{n} \summ i n \lpit it = \f{1}{C_r} \hat G(\Wt t).\]
The final inequality $(iv)$ follows since the step-size $\alpha$ is small enough by Assumption~\ref{a:stepsize}.  We therefore have shown that the unnormalized margin increases as claimed in part (b) of this lemma. 

\paragraph{Part (c): margin at time $t=1$.} Note that by~\eqref{eq:network.output.bounded.by.one}, $|f(x_k;\Wt 0)| \leq 1$ so that $$\lpit k0  = \frac{1}{1+\exp(f(x_k;\Wt 0))}\geq 1/(1+e)\geq 1/4,$$ and in particular we have
\begin{equation} \hat G(\Wt 0) \geq \f{1}{4},\qquad \text{and}\qquad \max_{i,j} \f{ \lpit i0}{\lpit j0} \leq 4.\label{eq:time0.intermediate}
\end{equation}
We thus have,
\begin{align*}
y_k f(x_k;\Wt 1) &= y_k f(x_k;\Wt 1) - y_k f(x_k; \Wt 0) + f(x_k;\Wt 0) \\
&\geq y_k f(x_k;\Wt 1) - y_k f(x_k; \Wt 0) - |f(x_k;\Wt 0)| \\
&\overset{(i)}\geq \frac{ \gamma^2 \alpha p}{16 C_1 n} - 2 C_0 \sinit \sqrt{C_1 p} (\sqrt m + \sqrt p)\\
&\overset{(ii)}\geq \frac{ \gamma^2 \sinit \sqrt{m} p^{3/2}}{16 C_1 n } - 2 C_0 \sinit \sqrt{C_1 p}(\sqrt m + \sqrt p) \\
&= \frac{ \gamma^2 \sinit \sqrt{m} p^{3/2}}{16 C_1 n } \l[ 1 - \f{ 32 C_0 C_1^{3/2} \gamma^{-2}  n (\sqrt m + \sqrt p)}{ \sqrt m p^{3/2} } \r] \\
&\overset{(iii)}\geq \f{ \gamma^2 \sinit \sqrt{m} p^{3/2}}{32 C_1 n}.
\end{align*}
The first term in inequality $(i)$ uses the lower bound provided in part (b) of this lemma as well as~\eqref{eq:time0.intermediate}, while the second term uses the upper bound on $|f(x_k;\Wt 0)|$ in~\eqref{eq:network.output.bounded.by.one}.  Inequality $(ii)$ uses Assumption~\ref{a:sinit} so that $\alpha \geq \sinit \sqrt{mp}$.  The final inequality $(iii)$ uses Assumption~\ref{a:dimension} so that $p\gg n^2$.  
\end{proof}

We now proceed with the proof of the loss ratio bound.  
\begin{proof}[Proof of Lemma~\ref{lemma:loss.ratio}]
In order to show that the ratio of the sigmoid losses $g(\cdot)$ is bounded, it suffices to show that the ratio of exponential losses $\exp(-(\cdot))$ is bounded, since by Fact~\ref{fact:logistic.loss.ratio},
\begin{equation}\label{eq:loss.ratio.logistic.exponential}
\max_{i,j=1,\dots, n} \f{g(y_i f(x_i; \Wt t))}{g(y_j f(x_j; \Wt t))} \leq \max\l( 2, 2 \cdot \max_{i,j=1,\dots, n} \f{\exp(-y_i f(x_i; \Wt t))}{\exp(-y_j f(x_j; \Wt t))}\r).
\end{equation}

Thus in the remainder of the proof we will show that the ratio of the exponential losses is bounded by an absolute constant.  Note that we have already shown in Lemma~\ref{lemma:nn.interpolates.time1} that the exponential loss ratio is bounded at initialization, thus we only need to show the result for times $t\geq 1$.

We now claim by induction that 
\begin{equation} \label{eq:exp.loss.ratio.induction}
\text{for all $t\geq 1$,}\qquad \max_{i,j=1,\dots, n} \f{ \exp(- y_i f(x_i;\Wt t))}{\exp(-y_j f(x_j; \Wt t))} \leq \frac{8 C_1^2}{\gamma^2} .
\end{equation}

Without loss of generality, it suffices to consider how the exponential loss ratio of the first sample to the second sample changes.  To this end, let us denote
\[ A_t := \f{ \exp(-y_1 f(x_1; \Wt t))}{\exp(-y_2 f(x_2; \Wt t))}.\]

\paragraph{Base case $t=1$.}  Intuitively, the base case holds because the loss ratio is small at initialization (due to Lemma~\ref{lemma:nn.interpolates.time1}) and the step-size is small by Assumption~\ref{a:stepsize}, so the loss ratio cannot increase too much after one additional step.  More formally, by Lemma~\ref{lemma:exp.loss.ratio} we have,
\begin{align*}
A_{1} &\leq A_0 \exp\left(-\frac{ \lpit 20 \alpha \gamma^2 p}{C_1 n}\left(\f{\lpit 10}{\lpit 20}  -\frac{C_1^2}{\gamma^{2}}   \right)\right)
    \exp\left(2C_1    \alpha \left(\lVert \mu\rVert^2+2 \sqrt{p \log(n/\delta)}\right) \hat G (W^{(0)}) \right)\\
&\overset{(i)}\leq A_0 \exp\left(\frac{ \lpit 20 C_1 \alpha p}{n} \right)
    \exp\left(2C_1    \alpha \left(\lVert \mu\rVert^2+ 2 \sqrt{p \log(n/\delta)} \right) \hat G (W^{(0)}) \right)\\
&\overset{(ii)} \leq  A_0 \exp\left(\frac{ C_1 \alpha p}{n} 
+ 2C_1  \alpha \left(\lVert \mu\rVert^2+ 2 \sqrt{p \log(n/\delta)} \right)\right)\\
&\overset{(iii)} \leq \exp(2) \cdot \exp\left(0.1\right)\leq 9.
\end{align*}
Inequalities~$(i)$ and $(ii)$ both use that $0\leq \lpit i0\leq 1$.  The first term in Inequality~$(iii)$ uses the upper bound on the exponential loss ratio at time $0$ given in Lemma~\ref{lemma:nn.interpolates.time1}, $A_0 \le \exp(2)$.  The second term in inequality $(iii)$ uses the assumption on the step size~\ref{a:stepsize} and that $p\gg n\snorm{\mu}^2$ by Assumption~\ref{a:dimension}.  This shows that the exponential loss ratio at time $1$ is at most $9$, which is at most $8C_1^2/\gamma^2$ for $C_1>2$ and $\gamma\leq 1$.   This completes the base case.

\paragraph{Induction step.} We now return to the induction step and assume the induction hypothesis holds for every time $\tau=1, \dots, t$, $A_\tau \le 8 C_1^2/\gamma^2$. Our task is to show is to show that the hypothesis holds at time $t+1$, that is, to show $A_{t+1}\leq 8 C_1^2/\gamma^2$. 

First, by Lemma~\ref{lemma:nn.interpolates.time1}, we know that for every time from time $\tau=1,\dots, t-1$, the unnormalized margin for each sample increased, and so by part (c) of that lemma we thus have
\begin{equation}\label{eq:nn.interpolates.with.loss.ratio}
\text{for all $k\in [n]$ and $\tau=1, \dots, t$,}\qquad y_k f(x_k;\Wt \tau) > 0.
\end{equation}
By Fact~\ref{fact:logistic.loss.ratio}, this means the ratio of exponential losses is at most twice the ratio of the sigmoid losses, which will be used in the analysis below.  

We now consider two cases: 
\begin{enumerate}
    \item If the ratio $\lpit 1t/\lpit 2t$ is relatively small, in this case we will show that the exponential loss ratio will not grow too much for small enough step-size $\alpha$.
    \item If the ratio $\lpit 1t / \lpit 2t$ is relatively large, then the first exponential term in~\eqref{eq:loss.ratio.intermediate} will dominate and cause the exponential loss ratio to contract. 
\end{enumerate}

\paragraph{Case 1 ($\lpit 1t / \lpit 2t \leq \frac{2 C_1^2}{\gamma^2} $):}  By Lemma~\ref{lemma:exp.loss.ratio}, we have
\begin{align*}
A_{t+1} &\leq A_t \exp\left(-\frac{ \lpit 2t \alpha \gamma^2 p}{C_1 n}\left(\f{\lpit 1t}{\lpit 2t}  -\frac{C_1^2}{\gamma^{2}}   \right)\right)
    \exp\left(2C_1    \alpha \left(\lVert \mu\rVert^2+2 \sqrt{p \log(n/\delta)}\right) \hat G (W^{(t)}) \right)\\
&\overset{(i)}\leq A_t \exp\left(\frac{ \lpit 2t C_1 \alpha p}{n} \right)
    \exp\left(2C_1    \alpha \left(\lVert \mu\rVert^2+ 2 \sqrt{p \log(n/\delta)} \right) \hat G (W^{(t)}) \right)\\
&\overset{(ii)} \leq  A_t \exp\left(\frac{ C_1 \alpha p}{n} \right)
    \exp\left(2C_1    \alpha \left(\lVert \mu\rVert^2+ 2 \sqrt{p \log(n/\delta)} \right)\right)\\
&\overset{(iii)} \leq 2\f{ \lpit 1t}{\lpit 2t} \exp\left(\frac{ C_1 \alpha p}{n} \right)
    \exp\left(2C_1    \alpha \left(\lVert \mu\rVert^2+ 2 \sqrt{p \log(n/\delta)} \right) \right)\\
&=2\f{ \lpit 1t}{\lpit 2t} \exp\left(C_1 \alpha  \l( \f{p}{n} + 2 \snorm{\mu}^2 + 4 \sqrt{p\log(n/\delta)} \r) \right) \\
&\overset{(iv)}\leq \f{ 4 C_1^2}{\gamma^2} \exp\left(C_1 \alpha  \l( \f{p}{n} + 2 \snorm{\mu}^2 + 4 \sqrt{p\log(n/\delta)} \r) \right) \\
&\overset{(v)}\leq \frac{4 C_1^2   \exp(1/8)}{\gamma^2} \leq \frac{8C_1^2}{ \gamma^{2}}.
\end{align*}
In $(i)$ and $(ii)$ we use that $0 \leq \lpit it \leq 1$.  In $(iii)$, we use~\eqref{eq:nn.interpolates.with.loss.ratio} so that $y_k f(x_k;\Wt t) > 0$ for all $k$.  In particular, by Fact~\ref{fact:logistic.loss.ratio} this means that the ratio of exponential losses is at most twice the ratio of the sigmoid losses.  In $(iv)$, we use the Case 1 assumption that $\lpit 1t / \lpit 2t\le 2C_1^2/\gamma^2$.  Finally, in $(v)$, we take $C>1$ sufficiently large so that by the upper bound on the step-size given in Assumption~\ref{a:stepsize}, we have,
\[ C_1 \alpha \l( \f{ p}{n} + 2 \snorm{\mu}^2 + 4 \sqrt{p\log(n/\delta)} \r) \leq \f 1 {Hn} + \f{6}{C_1 H} \leq \f 18,\]
where we have used Assumption~\ref{a:dimension} and assumed without loss of generality that $H\geq 1$.

\paragraph{Case 2 ($\lpit 1t / \lpit 2t > \frac{2C_1^2}{\gamma^{2}}$):} Again using Lemma~\ref{lemma:exp.loss.ratio}, we have that
\begin{align*}
A_{t+1}&\leq A_t \cdot  \exp\l(-\frac{ \lpit 2t \alpha \gamma^2 p}{C_1  n}\l (\f{ \lpit 1t}{\lpit 2t} - \frac{C_1^2}{\gamma^2}  \r)\r) \cdot \exp\l( 2C_1\alpha \l(\snorm{\mu}^2 + 2\sqrt{p \log(n/\delta)}\r) \hat G(\Wt t)\r)\\
&= A_t\exp\left(-\frac{ \lpit 2t \alpha \gamma^2 p}{C_1 n}\left(\f{\lpit 1t}{\lpit 2t}  - \frac{C_1^2}{ \gamma^{2}} \right)\right)\\&\qquad \times \exp\left(2 C_1 \alpha  \left(\lVert \mu\rVert^2+ 2\sqrt{p \log(n/\delta)} \right) \lpit 2t \cdot \f 1 n \summ i n \f{\lpit it}{\lpit 2t} \right) \\
&\overset{(i)} \leq A_t\exp\left(-\frac{ \lpit 2t \alpha \gamma^2 p}{C_1 n}\left(\f{\lpit 1t}{\lpit 2t}  - \frac{C_1^2}{ \gamma^{2}} \right)\right)\\ &\qquad \times \exp\left(2 \lpit 2t C_1 \alpha  \left(\lVert \mu\rVert^2+ 2\sqrt{p \log(n/\delta)} \right)  \cdot \max\left\{2, \frac{16 C_1^2}{ \gamma^{2}}\right\} \right) \\
&\overset{(ii)}= A_t\exp\left(- \lpit 2t \alpha \l[ \f{ \gamma^2 p}{C_1n} \l( \f{\lpit 1t}{\lpit 2t} - \frac{C_1^2}{ \gamma^{2}} \r) - \frac{32C_1^3}{\gamma^2}  \l(  \snorm{\mu}^2 +  2\sqrt{p \log(n/\delta)} \r) \r]  \right) \\
&\overset{(iii)}\leq   A_t\exp\left(-\lpit 2t \alpha \l[ \f{ C_1 p}{n} - \frac{32C_1^3}{\gamma^2} \l( \snorm{\mu}^2 + 2\sqrt{p \log(n/\delta)} \r) \r]  \right) \\
&\overset{(iv)}\leq A_t\leq \frac{8C_1^2}{ \gamma^{2}}.
\end{align*}
In $(i)$ we use the induction hypothesis that $A_t \leq  8 C_1^2/\gamma^2$ together with Fact~\ref{fact:logistic.loss.ratio}.  Equality $(ii)$ uses that $C_1>1$ and that $\gamma\leq 1$.  In $(iii)$, we use the Case 2 assumption that $\lpit 1t / \lpit 2t \geq 2 C_1^2/ \gamma^{2}$.  Finally, in $(iv)$, we use Assumption~\ref{a:dimension} so that we have $p \geq C n\lv \mu\rv^2\ge  \frac{128 C_1^2}{\gamma^2}  n  \snorm{\mu}^2$ and that $p \geq Cn^2 \log(n/\delta)\ge \left(\frac{128 C_1^2 }{\gamma^2} n  \sqrt{\log(n/\delta)}\right)^2$ and also the fact that $\lpit 2t\ge 0$.   

This completes the induction that for all times $t \geq 0$, the ratio of the exponential losses is at most $8C_1^2/\gamma^{2}$.  Using~\eqref{eq:loss.ratio.logistic.exponential} completes the proof. 
\end{proof}

\subsection{Proof of Lemma~\ref{lemma:margin.test.sample}}\label{ss:margin_increase_lemma}
Let us restate and prove the lemma. 
\marginincrease*
\begin{proof}
Using Lemma~\ref{lemma:nn.smooth}, if we define the quantity
\[ \xi_{i, z, s} := \f 1 m \summ j m \phi'(\sip{\wt s_j}{x_i}) \phi'(\sip{\wt s_j}{\mu+z}), \]
then recalling the notation $\lpit is := -\ell'\big(y_i f(x_i;\Wt s)\big) \in (0,1)$ we have,
\begin{align*}
    \E[ (f(\mu+z; \Wt {s+1}) - f(\mu+z;\Wt s))] &\geq \f \alpha n \summ i n \lpit is \l[ \E[\xi_{i,z,s} \sip{y_ix_i}{\mu+z}] - \f{ H C_1 p \alpha}{2 \sqrt m} \E[\snorm{\mu+z}^2] \r]. \numberthis \label{eq:margin.increase.time.s.mu.cluster}
\end{align*}

As for the $-f(-\mu+z;\Wt t)$ term, again Lemma~\ref{lemma:nn.smooth} states that if we define the quantity
\[ \xi_{i, z, s}' := \f 1 m \summ j m \phi'(\sip{\wt s_j}{x_i}) \phi'(\sip{\wt s_j}{-\mu+z}), \]
then
\begin{align*}
    \E[ -f(-\mu+z; \Wt {s+1}) - (-f(-\mu+z;\Wt s)))] &\geq \f \alpha n \summ i n \lpit is \l[ \E[ \xi_{i,z,s}' \sip{y_ix_i}{\mu-z}] - \f{ H C_1 p \alpha}{2 \sqrt m} \E[\snorm{\mu-z}^2] \r]. \numberthis \label{eq:margin.increase.time.s.minusmu.cluster}
\end{align*}

We first show the bound for $x=\mu+z$. Our goal is to show that the contribution of the terms involving $\xi_{i,z,s}
$ will be large and positive and will dominate the terms involving the $\snorm{\mu-z}^2$. In particular, using linearity of expectation and bi-linearity of the inner product,
\begin{align*}
    &\qquad \f \alpha n \summ i n \lpit is \E[\xi_{i,z,s} \sip{y_i x_i}{\mu+z} ] \\
    &=\f \alpha n \summ i n \lpit is \E \l[ \f 1 m \summ j m \phi'(\sip{\wt s_j}{x_i}) \phi'(\sip{\wt s_j}{\mu+z}) \sip{y_i x_i}{\mu+z} \r]\\
    &= \f \alpha {nm} \summ j m\ip{  \summ i n \lpit is \phi'(\sip{\wt s_j}{x_i}) y_i x_i}{\E[\phi'(\sip {\wt s_j}{\mu+z})( \mu+z)]} \\
    &= \f \alpha {nm} \summ j m \Bigg[ \ip{  \summ i n \lpit is \phi'(\sip{\wt s_j}{x_i}) y_i x_i}{\mu \E[\phi'(\sip {\wt s_j}{ \mu + z})] }\\
    &\qquad \qquad + \ip{  \summ i n \lpit is \phi'(\sip{\wt s_j}{x_i}) y_i x_i}{\E[\phi'(\sip {\wt s_j}{ \mu + z}) z]} \Bigg] \numberthis \label{eq:xi.izt.term.decomposition}
\end{align*}
The last equality uses that $z$ does not depend on $\mu$. 
We will show that the first term appearing in the brackets is large and positive, while the second term in the brackets is small in absolute value relative to this term.  For the latter term, we will use the near-orthogonality of the $\{y_i x_i\}$.  In particular we have,
\begin{align*}
    &\qquad\norm{ \summ i n \lpit is \phi'(\sip{\wt s_j}{x_i}) y_i x_i}^2 \\
    &= \summ i n \snorm{\lpit is \phi'(\sip{\wt s_j}{x_i}) y_i x_i}^2  + \sum_{i\neq k} \ip{\lpit is \phi'(\sip{\wt s_j}{x_i}) y_i x_i}{\lpit ks \phi'(\sip{\wt s_j}{x_k}) y_k x_k} \\
    &\overset{(i)} \leq \summ i n (\lpit is)^2 \snorm{x_i}^2 + \sum_{i\neq k} \lpit is \lpit ks |\sip{x_i}{x_k}| \\
    &\overset{(ii)} \leq \summ i n (\lpit is)^2 \cdot C_1 p + \sum_{i\neq k} \lpit is \lpit ks C_1 \l( \snorm{\mu}^2 + \sqrt{p\log(n/\delta)} \r) \\
    &\overset{(iii)} \leq C_1 C_r^2 n p \hat G(\Wt s)^2 + C_1 C_r^2 n^2 \l( \snorm \mu^2 + \sqrt{p\log(n/\delta)}\r) \hat G(\Wt s)^2 \\
    &\overset{(iv)} \leq 2C_1 C_r^2 n p \hat G(\Wt s)^2. \numberthis \label{eq:normsq.sumyixi.term.ub}
\end{align*}
The first inequality uses that $|\phi'|\leq 1$.  The second uses Lemma~\ref{lemma:sample.facts}.  Inequality $(iii)$ uses that by Lemma~\ref{lemma:loss.ratio},
\begin{equation} \label{eq:any.loss.vs.min.loss}
\text{for all $k$,\, }\lpit ks \leq \max_i \lpit is \leq C_r \min_i \lpit is \leq C_r \hat G(\Wt s). 
\end{equation}
The final inequality $(iv)$ uses that $p \geq C(n\snorm \mu^2 \vee n^2 \log(n/\delta))$ by~\ref{a:dimension} and by taking $C$ large enough. This allows for us to bound
\begin{align*}
    &\qquad \l|\ip{  \summ i n \lpit is \phi'(\sip{\wt s_j}{x_i}) y_i x_i}{\E[\phi'(\sip {\wt s_j}{\mu + z})  z]}\r| \\
    &= \l| \E \l [ \ip{z \phi'(\sip {\wt s_j}{ \mu + z})}{ \summ i n \lpit is \phi'(\sip{\wt s_j}{x_i}) y_i x_i} \r] \r| \\
    &= \l| \E \l [\phi'(\sip {\wt s_j}{ \mu + z}) \ip{ z}{\summ i n \lpit is \phi'(\sip{\wt s_j}{x_i}) y_i x_i} \r] \r| \\
    &\leq  \E\l[ |\phi'(\sip {\wt s_j}{\mu +z})| \cdot \l| \ip{z}{ \summ i n \lpit is \phi'(\sip{\wt s_j}{x_i}) y_i x_i}\r| \r] \\
    &\overset{(i)} \leq \E\l[ \l| \ip{z}{\summ i n \lpit is \phi'(\sip{\wt s_j}{x_i}) y_i x_i}\r| \r]\\
    &\overset{(ii)} \leq C_0 \norm{\summ i n \lpit is \phi'(\sip{\wt s_j}{x_i}) y_i x_i}\\
    &\overset{(iii)} \leq  C_0 \sqrt{2C_1 C_r^2 n p} \hat G(\Wt s). \numberthis \label{eq:noise.term.for.generalization.ub}
\end{align*}
Inequality $(i)$ uses that $|\phi'(q)|\leq 1$, while inequality $(ii)$ uses that $\summ i n \lpit is \phi'(\sip{\wt s_j}{x_i}) y_i x_i$ is a quantity which does not depend on $z$, and that $z$ has sub-Gaussian norm at most one.  Inequality $(iii)$ uses~\eqref{eq:normsq.sumyixi.term.ub}.

The above calculation shows an upper bound for the absolute value of the second term of~\eqref{eq:xi.izt.term.decomposition}; we now show a lower bound for the first term.  We have,
\begin{align*}
     &\qquad\ip{  \summ i n \lpit is \phi'(\sip{\wt s_j}{x_i}) y_i x_i}{\E[\phi'(\sip {\wt s_j}{  \mu + z}) \mu]}\\
     &=   \summ i n \lpit is \phi'(\sip{\wt s_j}{x_i}) \ip{y_i x_i}{\E[\phi'(\sip {\wt s_j}{\mu + z}) \mu]} \\
     &\overset{(i)}=   \summ i n \lpit is \phi'(\sip{\wt s_j}{x_i}) \ip{y_i x_i}{\mu }\E[\phi'(\sip {\wt s_j}{\mu + z})]\\
     &=   \sum_{i\in \calC} \lpit is \phi'(\sip{\wt s_j}{x_i}) \ip{y_i x_i}{\mu }\E[\phi'(\sip {\wt s_j}{\mu + z})] +   \sum_{i\in \calN} \lpit is \phi'(\sip{\wt s_j}{x_i}) \ip{y_i x_i}{\mu }\E[\phi'(\sip {\wt s_j}{\mu + z})]\\
     &\overset{(ii)}\geq \sum_{i\in \calC} \gamma \cdot \f{ \snorm{\mu}^2}{2} \cdot \gamma \cdot \lpit is- \sum_{i\in \calN} 1 \cdot \f{ 3 \snorm \mu^2}{2} \cdot 1 \cdot \lpit is\\
     &= \f{ \gamma^2 \snorm{\mu}^2 }2 \l( \sum_{i\in \calC} \lpit is  - 3 \gamma^{-2}  \sum_{i\in \calN} \lpit is \r) \\
     &= \f{ \gamma^2 \snorm{\mu}^2 }2 \l( \summ i n \lpit is  - (3 \gamma^{-2}+1)  \sum_{i\in \calN} \lpit is \r) \\
     &\overset{(iii)}\geq \f{ \gamma^2 \snorm{\mu}^2 }2 \l( n \hat G(\Wt s)  - (3 \gamma^{-2} +1)|\calN| C_r \hat G(\Wt s) \r) \\
     &\overset{(iv)} \geq \f{ \gamma^2 \snorm \mu^2}{4} \cdot n \hat G(\Wt s). \numberthis \label{eq:xi.izt.term.1.lb}
\end{align*}
The equality $(i)$ uses that $z$ does not depend on $\mu$ and linearity of expectation, while $(ii)$ uses that $\gamma \leq \phi'(q)\leq 1$ for all $q$ and Lemma~\ref{lemma:sample.facts}. Inequality $(iii)$ uses~\eqref{eq:any.loss.vs.min.loss}, and inequality $(iv)$ uses Lemma~\ref{lemma:sample.facts} and by taking $C$ large enough in Assumption~\ref{a:noiserate}.

Putting~\eqref{eq:xi.izt.term.1.lb} and~\eqref{eq:noise.term.for.generalization.ub} into~\eqref{eq:xi.izt.term.decomposition} we get
\begin{align*}
    \f \alpha n \summ i n \lpit is \E[\xi_{i,z,s} \sip{y_i x_i}{\mu+z} ] &= \f \alpha {nm} \summ j m \Bigg[ \ip{  \summ i n \lpit is \phi'(\sip{\wt s_j}{x_i}) y_i x_i}{\E[\phi'(\sip {\wt s_j}{\tilde y \mu + z}) \mu]}\\
    &\qquad \qquad + \ip{  \summ i n \lpit is \phi'(\sip{\wt s_j}{x_i}) y_i x_i}{\E[\phi'(\sip {\wt s_j}{\tilde y \mu + z}) z]} \Bigg] \\
    &\geq \f{\alpha}{nm} \summ j m \l[ \f{ \gamma^2 \snorm \mu^2}{4} \cdot n \hat G(\Wt s) -  C_0 \sqrt{2C_1 C_r^2 n p} \hat G(\Wt s) \r]  \\
    &= \f{ \alpha \gamma^2 \snorm{\mu}^2}{4} \l( 1 -  4C_0 \gamma^{-2} \sqrt{ \f{ 2 C_1 C_r^2 p}{n\snorm \mu^4}} \r)\hat G(\Wt s). \numberthis \label{eq:margin.increase.time.s.mu.cluster.firstterm}
\end{align*}

Finally, the second term of~\eqref{eq:margin.increase.time.s.mu.cluster} can be bounded from above by using that $z$ has sub-Gaussian norm at most one, hence
\begin{align*}
    \E\snorm{\mu+z}^2 &= \snorm{\mu}^2 + 2\E\sip{\mu}{z} + \E\snorm{z}^2 =  \snorm{\mu}^2 + \E\snorm{z}^2 \leq \snorm \mu^2 + 3p \leq 4p,
\end{align*}
where the last inequality uses assumption~\ref{a:dimension}.  
This means
\begin{align*}
    \summ i n \lpit is \cdot \f{ H C_1 p \alpha}{2 \sqrt m} \E[\snorm{\mu+z}^2] \leq \f{2 H C_1 p^2 \alpha}{\sqrt m} \cdot n \hat G(\Wt s).
\end{align*}
Using this and~\eqref{eq:margin.increase.time.s.mu.cluster.firstterm}, we see that~\eqref{eq:margin.increase.time.s.mu.cluster} becomes
\begin{align*}
    &\qquad \E[ (f(\mu+z; \Wt {s+1}) - f(\mu+z;\Wt s))] \\
    &\geq \f \alpha n \summ i n \lpit is \l[ \E[\xi_{i,z,s} \sip{y_ix_i}{\mu+z}] - \f{ H C_1 p \alpha}{2 \sqrt m} \E[\snorm{\mu+z}^2] \r]\\
    &\geq \f{\alpha \gamma^2 \snorm{\mu}^2}{4} \l( 1 -  4C_0 \gamma^{-2} \sqrt{ \f{ 2 C_1 C_r^2 p}{n\snorm \mu^4}} \r)\hat G(\Wt s) -  \f{2 H C_1 p^2 \alpha^2}{\sqrt m} \hat G(\Wt s)\\
    &= \f{ \alpha \gamma^2 \snorm{\mu}^2}{4} \l( 1 -  4C_0 \gamma^{-2} \sqrt{ \f{ 2 C_1 C_r^2 p}{n\snorm \mu^4}} - \f{8 H C_1 p^2 \alpha}{ \gamma^2 \snorm{\mu}^2 \sqrt m} \r)\hat G(\Wt s)  \\
    &\overset{(i)} \geq \f{ \alpha \gamma^2 \snorm{\mu}^2}{8} \hat G(\Wt s).
\end{align*}
The final inequality follows by taking $\alpha$ small enough so that $8H C_1 p^2 \alpha / (2 \gamma^2\snorm{\mu}^2 \sqrt m) \leq 1/4$ via assumption~\ref{a:stepsize} and~\ref{a:norm.mu}, and by taking $C_\mu$ large enough in the lemma's assumption that $n \snorm \mu^4 \geq C_\mu p$.

It is clear that the same argument applies to the test example $-\mu+z$ as the only properties of $\xi_{i,z,s}$ that are used in the proof are that $\gamma \leq \phi'(q) \leq 1$ for all $q$, and hence swapping out $\xi_{i,z,s}$ for $\xi_{i, z, s}'$ will result in the same bounds. 
\end{proof}

\subsection{Proof of Lemma~\ref{l:refined_grad_norm_upper_bound}}\label{ss:refined_upper_bound}
We remind the reader of the statement of Lemma~\ref{l:refined_grad_norm_upper_bound}.
\parameternormbound*

\begin{proof} By the triangle inequality we have that
\begin{align*}
    \lv W^{(t)}\rv_F & = \left\lv W^{(0)} + \alpha \sum_{s=0}^{t-1} \nabla \hat L(\Wt s)\right\rv_F  \le \lv W^{(0)}\rv_F +  \alpha \sum_{s=0}^{t-1}\lv \nabla \hat L(\Wt s)\rv_F 
    . \numberthis \label{e:refined_upper_norm_bound_step1}
\end{align*}
Now observe that
\begin{align*}
    &\lv \nabla \hat{L}(W^{(s)})\rv_F^2\\&\qquad = \frac{1}{n^2} \left\lv \sum_{i=1}^n \lpit is y_i \nabla f(x_i;W^{(s)})\right\rv_F^2 \\
    & \qquad= \frac{1}{n^2} \left[\sum_{i=1}^n \left(\lpit is\right)^2 \left\lv \nabla f(x_i;W^{(s)}) \right\rv^2_F +\sum_{i\neq j \in [n]}\lpit is \lpit js y_iy_j\langle \nabla f(x_i;W^{(s)}),\nabla f(x_j;W^{(s)}) \rangle \right] \\
    & \qquad\le \frac{1}{n^2} \left[\sum_{i=1}^n \left(\lpit is \right)^2 \left\lv \nabla f(x_i;W^{(s)}) \right\rv^2_F +\sum_{i\neq j \in [n]} \lpit is \lpit js \left\lvert \langle \nabla f(x_i;W^{(s)}),\nabla f(x_j;W^{(s)}) \rangle\right\rvert  \right] \\
    &\qquad\overset{(i)}{\le} \frac{C_1}{n^2} \left[\sum_{i=1}^n \left(\lpit is \right)^2 p +\sum_{i\neq j \in [n]} \lpit is \lpit js \left(\lv \mu\rv^2 + \sqrt{p\log(n/\delta)}\right) \right] \\
    & \qquad\le  \frac{C_1}{n^2}\cdot  \max_{k \in [n]} \lpit ks \left[\sum_{i=1}^n \lpit is  p + n\sum_{i=1}^n \lpit is \left(\lv \mu\rv^2 + \sqrt{p\log(n/\delta)}\right) \right] \\
    & \qquad= \frac{C_1}{n^2} \left(p+n\lv \mu\rv^2+n\sqrt{p \log(n/\delta)}\right)\cdot \max_{k \in [n]} \lpit ks \left[\sum_{i=1}^n \lpit is  \right],
\end{align*}
where $(i)$ follows by Lemma~\ref{lemma:nn.grad.ip.identity}. Now note that since $p \ge Cn\lv \mu\rv^2$ and $p \ge Cn^2 \log(n/\delta)$ by Assumption~\ref{a:dimension}, we have that,
\begin{align*}
        \lv \nabla \hat{L}(W^{(s)})\rv_F^2 & \le \frac{3 C_1^2 p}{n} \left( \max_{k\in [n]}\lpit ks \right)\hat{G}(W^{(s)}).
\end{align*}
Next note that by the loss ratio bound in Lemma~\ref{lemma:loss.ratio} we have that
\begin{align*}
    \max_{k\in [n]}\lpit ks \le \frac{C_r}{n} \sum_{i=1}^n \lpit is = C_r \hat{G}(W^{(s)}).
\end{align*}
Plugging this into the previous inequality yields
\begin{align*}
    \lv \nabla \hat{L}(W^{(s)})\rv_F^2 \le \frac{3 C_1^2 C_r p}{n}\left(\hat{G}(W^{(s)})\right)^2.
\end{align*}
Finally, taking square roots, defining $C_2 := \sqrt{3 C_1^2 C_r}$ and applying this bound on the norm in Inequality~\eqref{e:refined_upper_norm_bound_step1} above we conclude that
\begin{align*}
    \lv W^{(t)}\rv_F \le \lv W^{(0)}\rv_F +C_2 \alpha \sqrt{\frac{p}{n}} \sum_{s=0}^{t-1}\hat{G}(W^{(s)}),
\end{align*}
establishing our claim.
\end{proof}

\subsection{Proof of Lemma~\ref{lemma:normalized.margin.lb}}\label{ss:normalized_margin_lb}
Let us restate the lemma for the reader's convenience.
\normalizedmarginbound*
\begin{proof}
Using the refined upper bound for the norm of the weights given in Lemma~\ref{l:refined_grad_norm_upper_bound}, we have that,
\begin{align*}
\snorm{\Wt t}_F 
&\leq \snorm{\Wt 0}_F + C_2 \alpha \sqrt{\f{p}{n}}  \summm s 0 {t-1} \hat G(\Wt s).\numberthis \label{eq:wt.fro.norm.w1}
\end{align*}
To complete the proof, we want to put together the bound for the unnormalized margin on clean samples given by Lemma~\ref{lemma:margin.test.sample} with the upper bound on the norm given in~\eqref{eq:wt.fro.norm.w1}.  We'll consider the $+\mu$ case and the $-\mu$ case follows identically.   By Lemma~\ref{lemma:margin.test.sample}, we know that for non-negative integers $s$,
\begin{align*}
    \E[f(\mu+z;\Wt {s+1}) - f(\mu+z;\Wt s)] \geq \f{ \alpha \gamma^2 \snorm \mu^2}{8} \hat G(\Wt s).
\end{align*}
Summing this from $s=0$ to $s={t-1}$ we get
\begin{align*}\numberthis \label{eq:margin.at.t.minus.margin.at.0.prelim}
    \E[f(\mu+z;\Wt t)] - \E[f(\mu+z;\Wt 0)] \geq \f{ \alpha \gamma^2 \snorm \mu^2}{8} \summm s 0 {t-1} \hat G(\Wt s).
\end{align*}
Now we'd like to show that the margin term from $t=0$ can be ignored.  To this end,
we first write
\begin{align*}
    |\E[f(\mu+z;\Wt 0)]| &\leq \E |f(\mu+z;\Wt 0)| \\
    &= \E\l| \summ j m a_j \phi(\sip{\wt 0_j}{\mu+z}) \r|\\
    &\leq \f{1}{\sqrt m} \summ j m \E|\phi(\sip{\wt 0_j}{\mu+z})| \\
    &\overset{(i)}\leq \f{1}{\sqrt m} \summ j m \E |\sip{\wt 0_j}{\mu+z}| \\
    &\leq \f 1 {\sqrt m} \summ j m ( |\sip{\wt 0_j}{\mu}| + \E|\sip{\wt 0_j}{z}| ). \numberthis \label{eq:margin.at.0.ub.prelim}
\end{align*}
Inequality $(i)$ uses that $|\phi(q)| \leq |q|$ for all $q\in \R$.  For the first term above, we have
\begin{align*}
    \f 1 {\sqrt m} \summ j m |\sip{\wt 0_j}{\mu}| &\overset{(i)}\leq \sqrt{\summ j m \sip{\wt 0_j}{\mu}^2} \\
    &= \snorm{\Wt 0 \mu}_2 \\
    &\leq \snorm{\Wt 0}_2 \snorm{\mu} \\
    &\overset{(ii)} \leq C_0 \sinit(\sqrt m + \sqrt p) \snorm \mu\\
    &\overset{(iii)} \leq C_0 \alpha \l( \f{1}{\sqrt p} + \f 1{\sqrt m} \r) \snorm \mu. \numberthis \label{eq:wj0.mu.term}
\end{align*}
Inequality $(i)$ uses Cauchy--Schwarz.  Inequality $(ii)$ uses Lemma~\ref{lemma:initialization.norm}, while $(iii)$ uses Assumption~\ref{a:sinit}.  For the second term in~\eqref{eq:margin.at.0.ub.prelim}, we have
\begin{align*}
        \f 1 {\sqrt m} \summ j m \E|\sip{\wt 0_j}{z}| &\overset{(i)}\leq \sqrt{\summ j m (\E |\sip{\wt 0_j}{z}|)^2} \\
    &\overset{(ii)}\leq \sqrt{\summ j m \E[\sip{\wt 0_j}{z}^2]} \\
    &\overset{(iii)} \leq \sqrt{\summ j m 2\snorm{\wt 0_j}^2} \\
    &= \sqrt 2 \snorm{\Wt 0}_F\\
    &\overset{(iv)} \leq \sqrt 2 \sinit \sqrt{mp}\\
    &\overset{(v)} \leq \sqrt 2 \alpha. \numberthis 
\end{align*}
Inequality $(i)$ again uses Cauchy--Schwarz.  Inequality $(ii)$ uses Jensen's inequality, and $(iii)$ uses that $z$ is mean zero and has sub-Gaussian norm at most one.   Inequality $(iv)$ uses Lemma~\ref{lemma:initialization.norm}, and $(v)$ uses Assumption~\ref{a:sinit}.  Substituting the previous display and~\eqref{eq:wj0.mu.term} into~\eqref{eq:margin.at.0.ub.prelim} we get
\begin{align*}
    |\E[f(\mu+z;\Wt 0)]| \leq C_0 \alpha \snorm \mu + \sqrt 2 \alpha \leq 2 C_0 \alpha \snorm \mu,
\end{align*}
where we use assumption~\ref{a:norm.mu} and assume without loss of generality that $C_0>\sqrt 2$.  Using this in~\eqref{eq:margin.at.t.minus.margin.at.0.prelim} we get
\begin{align*}
     \E[f(\mu+z;\Wt t)] &\geq \f{ \alpha \gamma^2 \snorm \mu^2}{8} \summm s 0 {t-1} \hat G(\Wt s) - 2 C_0\alpha \snorm \mu \\
     &= \f{ \alpha \gamma^2 \snorm{\mu}^2}{8} \l( -\f{ 16 C_0\gamma^{-2}}{\snorm{\mu}} + \f 12 \hat G(\Wt 0) + \f 1 2 \hat G(\Wt 0) + \summm s 1 {t-1} \hat G(\Wt s) \r) \\
     &\overset{(i)}\geq \f{ \alpha \gamma^2 \snorm{\mu}^2}{8} \l( -\f{ 16 C_0\gamma^{-2}}{\sqrt {C \log(n/\delta)}} + \f 12 \hat G(\Wt 0) + \f 1 2 \hat G(\Wt 0) + \summm s 1 {t-1} \hat G(\Wt s) \r) \\
     &\overset{(ii)}\geq \f{ \alpha \gamma^2 \snorm{\mu}^2}{8} \l( \f 1 2 \hat G(\Wt 0) + \summm s 1 {t-1} \hat G(\Wt s) \r) \\
     &\geq \f{ \alpha \gamma^2 \snorm \mu^2}{16} \summm s 0 {t-1} \hat G(\Wt s). \numberthis \label{eq:unnormalized.margin.final}
\end{align*}
Inequality $(i)$ uses assumption~\ref{a:norm.mu}.  Inequality $(ii)$ follows by noting that $\hat G(\Wt 0)\geq 1/4$ and by taking $C$ to be a large enough absolute constant.  That $\hat G(\Wt 0)\geq 1/4$ follows by equation~\eqref{eq:network.output.bounded.by.one}: since $|f(x_i; \Wt 0)| \leq 1$ for all $i$ the definition of $\ell$ implies $-\ell'(q)\geq 1/4$ for $|q| \leq 1$.

We provide one final auxiliary calculation before showing the lower bound on the normalized margin.  By the previous paragraph's argument, we have $\hat G(\Wt 0) \geq 1/4$. 
Using this along with Lemma~\ref{lemma:initialization.norm}, we have that
\begin{equation}\label{eq:init.norm.vs.loss}
\quad \snorm{\Wt 0}_F \leq 2 \sinit \sqrt{mp} \leq 2 \alpha \leq \alpha \sqrt{C_1p/n} \hat G(\Wt 0),
\end{equation}
where we have used the assumption~\ref{a:sinit} that $\sinit \sqrt{mp}\leq \alpha$ and that Assumption~\ref{a:dimension} implies $p/n$ is larger than some fixed constant.  

With this in hand, we can calculate a lower bound on the normalized margin as follows. We consider two disjoint cases.

\paragraph{Case 1 ($\snorm{\Wt t}_F\leq 2 \snorm{\Wt 0}_F$):}   In this case, by using ~\eqref{eq:unnormalized.margin.final} we have that,
\begin{align*}
\f{ \E [f(\mu+z; \Wt t)]}{\snorm{\Wt t}_F} &\geq \f{ \alpha \gamma^2 \snorm{\mu}^2 \summm s 0 {t-1} \hat G(\Wt s)}{16 \cdot 2\snorm{\Wt 0}_F} \\
&\overset{(i)}{\geq} \f{ \alpha \gamma^2 \snorm{\mu}^2 \summm s 0 {t-1} \hat G(\Wt s)}{32\alpha \sqrt{C_1 p/n} \hat G(\Wt 0)} \\
&\overset{(ii)}{\geq} \f{\gamma^2 \snorm{\mu}^2 \sqrt n}{32\sqrt{C_1 p}}
\end{align*} 
where $(i)$ uses \eqref{eq:init.norm.vs.loss} and $(ii)$ uses that $\sum_{s=0}^{t-1}G(\Wt s) \ge G(\Wt 0)$. This completes the proof in this case.

\paragraph{Case 2 ($\snorm{\Wt t}_F> 2 \snorm{\Wt 0}_F$):}
   By~\eqref{eq:wt.fro.norm.w1}, we have the chain of inequalities,
\[ 2\snorm{\Wt 0}_F < \snorm{\Wt t}_F \leq \snorm{\Wt 0}_F + C_2 \alpha \sqrt{\f pn} \summm s 0 {t-1}\hat G(\Wt s).\]
In particular, we have $C_2 \alpha \sqrt{p/n} \summm s 0 {t-1} \hat G(\Wt s) > \snorm{\Wt 0}_F$, and so using the preceding inequality and~\eqref{eq:unnormalized.margin.final} we get,
\begin{align*}
    \f{ \E[f(\mu+z; \Wt t)]}{\snorm{\Wt t}_F} &\geq \f{ \alpha \gamma^2 \snorm{\mu}^2 \summm s 0 {t-1} \hat G(\Wt s)}{16\big(\snorm{\Wt 0}_F + C_2 \alpha \sqrt{p/n} \summm s 0 {t-1} \hat G(\Wt s)\big)} \\
    &\geq \f{ \alpha \gamma^2 \snorm{\mu}^2 \summm s 0 {t-1} \hat G(\Wt s)}{32 C_2 \alpha \sqrt{p/n} \summm s 0 {t-1} \hat G(\Wt s)} \\
    &= \f{\gamma^2 \snorm{\mu}^2 \sqrt n }{32 C_2\sqrt p},
\end{align*}
completing the proof.  
\end{proof}

\subsection{Proof of Lemma~\ref{lemma:optimization.guarantee}}\label{ss:optimization_guarantee}
\gradientnormlowerbound*
\begin{proof}
In order to show a lower bound for $\snorm{\nabla \hat L(\Wt t)}_F = \sup_{U: \lv U\rv_F=1}\langle -\nabla \hat L(\Wt t), U \rangle $, it suffices to construct a matrix $V$ with Frobenius norm at most one such that
$\sip{-\nabla \hat L(\Wt t)}{V}$ is bounded from below by a positive constant.  To this end, let $V\in \R^{m\times p}$ be the matrix with rows
\begin{equation}\label{eq:V.matrix}
 v_j =  a_j \mu / \snorm{\mu}.
\end{equation}
Then $\snorm{V}_F=1$ (since $a_j = \pm 1/\sqrt{m}$), and we have for any $W\in \R^{m\times d}$,
\begin{align}
    \sip{\nabla f(x_i; W)}{V} &= \summ jm a_j \phi'(\sip{w_j}{x}) \sip{v_j}{x} = \ip{\f{ \mu}{\snorm \mu}}{x} \f 1 m \summ i m \phi'(\sip{w_j}{x}).\label{eq:gradient.margin.intermediate}
\end{align}
Now, by Events~\eqref{eq:mu.dot.xk.clean} and~\eqref{eq:mu.dot.xk.noisy}, we have that
\begin{equation}
\begin{cases}y_i \sip{\mu}{x_i} \geq \f 12 \snorm{\mu}^2, & i\in \calC,\\
|\sip{\mu}{x_i}| \leq \f 32 \snorm{\mu}^2, &i\in \calN.\end{cases}
\end{equation}
Since $\phi'(z)\geq \gamma>0$ for all $z$,~\eqref{eq:gradient.margin.intermediate} implies we have the following lower bound for any $W\in \R^{m\times d}$,
\begin{equation}\label{eq:gradient.margin}
y_i \sip{\nabla f(x_i; W)}{V} \geq \begin{cases}  \f{\gamma}2 \snorm{\mu}, &\quad i\in \calC,\\ 
 -\f{3}{2} \snorm{\mu}, &\quad i\in \calN.
\end{cases} 
\end{equation} 
This allows for a lower bound on $\sip{-\hat \nabla L(\Wt s)}{V}$, since
\begin{align*}
\sip{-\hat \nabla L(\Wt s)}{V} &= \f 1 n \summ i n \lpit is y_i \sip{\nabla f(x_i; \Wt s)}{V} \\
&\overset{(i)}\geq \f 1 n \sum_{i\in \calC} \lpit is \cdot \f{\gamma}{2} \snorm{\mu} - \f 1 n \sum_{i\in \calN} \lpit is \cdot \f 32 \snorm{\mu} \\
&= \f {\gamma \snorm{\mu}}2 \l[ \hat G(\Wt s) - \l( 1 + \frac{3}{\gamma} \r) \f 1 n \sum_{i\in \calN} \lpit is \r]\\
&\overset{(ii)}\geq \f {\gamma  \snorm{\mu}} 2 \l[ \hat G(\Wt s) - \left(1 + \frac{3}{\gamma}\right) \cdot 2C_r \eta \hat G(\Wt s) \r] \\
&\overset{(iii)}\geq \f {\gamma  \snorm{\mu}} 4 \hat G(\Wt s).\numberthis \label{eq:intermediate.norm.lower.bound}
\end{align*}
Inequality $(i)$ uses~\eqref{eq:gradient.margin}, while $(ii)$ uses Lemma~\ref{lemma:loss.ratio}, which implies for every $i\in [n]$, $\lpit it \leq C_r \hat G(\Wt t)$.  
Finally, inequality $(iii)$ above uses Assumption~\ref{a:noiserate} so that the noise rate satisfies $\eta \leq 1/C\le  [4C_r(1 + 3/\gamma)]^{-1}$.   We can therefore derive the following lower bound on the norm of the gradient,
\begin{align}\label{eq:proxy.pl}
    \text{for any $t \geq 0$,}\quad \snorm{\nabla \hat L(\Wt t)}_F &\geq \sip{\nabla \hat L(\Wt t)}{-V} \geq \f{ \gamma \norm{\mu} \hat G(\Wt t)}{4}.
\end{align}
We notice that the inequality of the form $\snorm{\hat \nabla L(W)}\geq c \hat G(W)$ is a proxy PL inequality, where the proxy loss function is $\hat G(W)$~\citep{frei2021proxyconvex}.  We can therefore mimic the smoothness-based proof of~\citet[Theorem 3.1]{frei2021proxyconvex} to show that $\hat G(\Wt {T-1}) \leq \eps$ for $T = \Omega(\eps^{-2})$.  By Lemma~\ref{lemma:loss.smooth}, the loss $\hat L(W)$ has $C_1p(1 + H/\sqrt m)$-Lipschitz gradients.  In particular, we have
\begin{equation}\label{eq:loss.smoothness.intermediate}
    \hat L(\Wt {t+1}) \leq \hat L(\Wt t) - \alpha \snorm{\nabla \hat L(\Wt t)}_F^2 + C_1p\max\left\{1,\frac{H}{\sqrt m}\right\} \alpha^2 \snorm{\nabla \hat L(\Wt t)}_F^2.
\end{equation}
In particular, since Assumption~\ref{a:stepsize} requires $\alpha \leq 1/\left(2 \max\left\{1,\frac{H}{\sqrt m}\right\} C_1^2 p^{2}\right)$, we have that
\begin{equation*}
\snorm{ \nabla \hat L(\Wt t)}_F^2 \leq \f 2{\alpha} \l[ \hat L(\Wt {t}) - \hat L(\Wt {t+1})\r].
\end{equation*}
Telescoping the above sum and scaling both sides by $1/T$, we get for any $T\geq 1$,
\begin{align*} \numberthis \label{eq:intermediate.regret.bound.cale}
\f{\gamma^2 \snorm \mu^2}{16}  \f 1 {T} \summm t 0 {T-1} \hat G(\Wt t)^2 &\overset{(i)}\leq \f 1 {T} \summm t 0 {T-1} \snorm{\nabla \hat L(\Wt t)}_F^2 \leq \f{ 2\hat L(\Wt 0)}{\alpha T},
\end{align*}
where 
inequality $(i)$ uses the proxy PL inequality~\eqref{eq:proxy.pl}.  
Finally, note that by Lemma~\ref{lemma:loss.ratio}, for all $t\geq 0$ we have $\max_{i,j} \nicefrac {\lpit it}{\lpit jt}\leq C_r$.  Thus by Lemma~\ref{lemma:nn.interpolates.time1} part (b), we know that the unnormalized margin $y_k f(x_k;\Wt t)$ is an increasing function of $t$ for each $k\in [n]$.  Since $g = -\ell'$ is decreasing, this implies 
the loss $\hat G(\Wt t)$ is a decreasing function of $t$, and hence $\hat G(\Wt t)^2$ is a decreasing function of $t$.  Therefore, by~\eqref{eq:intermediate.regret.bound.cale},
\begin{equation}
\hat G(\Wt {T-1})^2 = \min_{t<T} \hat G(\Wt t)^2 \leq \f 1 T \summm t 0 {T-1} \hat G(\Wt t)^2 \leq \f{32 \hat L(\Wt 0)}{\gamma^2 \snorm{\mu}^2 \alpha T} \leq \eps^2/4,
\end{equation}
where in the last inequality we use that $T \geq 128 \hat L(\Wt 0)/\left( \gamma^2 \snorm{\mu}^2 \alpha \eps^{2}\right)$.  The proof is completed by noting that $\ind(z \leq 0)\leq  -2\ell'(z)$ and that $\hat G(\Wt {T-1})$ is positive. 
\end{proof}
\section{Non-NTK results, Proof of Proposition~\ref{proposition:non.ntk}}\label{app:prop_proof}
For the reader's convenience, we restate Proposition~\ref{proposition:non.ntk} here.
\ntklower*
\begin{proof}
We construct a lower bound on $\snorm{\Wt 1 -\Wt 0}_F$ using the variational formula for the norm, namely $\snorm{\Wt 1 - \Wt 0}_F \geq \sip{\Wt 1 - \Wt 0}{V}$ for any matrix $V$ with Frobenius norm at most 1.  By definition,
\[ \sip{\Wt 1 - \Wt 0}{V} = \alpha \sip{-\nabla \hat L(\Wt 0)}{V}.\]
By Lemmas~\ref{lemma:initialization.norm} and~\ref{lemma:sample.facts}, a good run occurs with probability at least $1-2\delta$.  On a good run we can use the results in Lemma~\ref{lemma:optimization.guarantee}.  In particular, with the choice of $V$ given in eq.~\eqref{eq:V.matrix}, we have,
\begin{align*}
    \snorm{\Wt 1 - \Wt 0}_F &\geq \sip{\Wt 1 - \Wt 0}{V} \\
    &= \alpha \sip{-\nabla \hat L(\Wt 0)}{V} \\
    &\overset{(i)}\geq \f{ \alpha \gamma \snorm{\mu}}{4} \hat G(\Wt 0) \\
    &\overset{(ii)}\geq \f{\alpha \gamma \snorm{\mu}}{24},
\end{align*}
where inequality $(i)$ uses eq.~\eqref{eq:proxy.pl} and the last inequality $(ii)$ uses that $\hat G(\Wt 0)\geq 1/6$ due to $|f(x_i;\Wt 0)| \leq 1$ (see~\eqref{eq:network.output.bounded.by.one}) and properties of $\ell'$.  Thus, by Lemma~\ref{lemma:initialization.norm}, we have
\begin{align*}
    \f{ \snorm{\Wt 1 - \Wt 0}_F}{\snorm{\Wt 0}_F} \geq \f{ \alpha \gamma \snorm{\mu}}{48 \sinit \sqrt{mp}} \geq \f{\gamma \snorm{\mu}}{48},
\end{align*}
where the last inequality uses Assumption~\ref{a:sinit}. 
\end{proof}

\printbibliography

\end{document}